\documentclass{article}
\usepackage{iclr2027_conference,times}

\usepackage{amsmath,amssymb,amsthm}
\usepackage{graphicx}
\usepackage{float}
\usepackage{wrapfig}
\usepackage{pifont}
\usepackage{makecell}
\usepackage{algorithm}
\usepackage{algpseudocode}
\usepackage{multirow}
\usepackage{booktabs}
\usepackage{xcolor}
\usepackage{hyperref}
\usepackage{url}
\usepackage{titletoc}

\newtheorem{proposition}{Proposition}

\title{Flow Matching Reinforcement for 3D Mesh Generation via Dynamic Homing Optimization}

\author{%
\parbox[t]{\dimexpr\textwidth-2\tabcolsep\relax}{%
\centering\normalfont
Zhen Zhou\textsuperscript{1,}\thanks{Equal contribution. \textsuperscript{$\dagger$}Corresponding authors. Email: wenw27958@gmail.com.}, Zhiwei Ning\textsuperscript{1,2,*}, Puhua Jiang\textsuperscript{1}, Sheng Zhang\textsuperscript{1}, Yifei Tang\textsuperscript{3},\\
Jie Yang\textsuperscript{2}, Xintong Han\textsuperscript{1}, Wei Liu\textsuperscript{2,$\dagger$}, Chunchao Guo\textsuperscript{1,$\dagger$}\\[0.4em]
\textsuperscript{1}Tencent Hunyuan, \textsuperscript{2}Shanghai Jiao Tong University\\
\textsuperscript{3}Communication University of China%
}%
}

\iclrfinalcopy

\fancypagestyle{arxivheader}{%
    \fancyhead{}%
    \lhead{%
        \raisebox{-0.2\height}{\includegraphics[height=10pt]{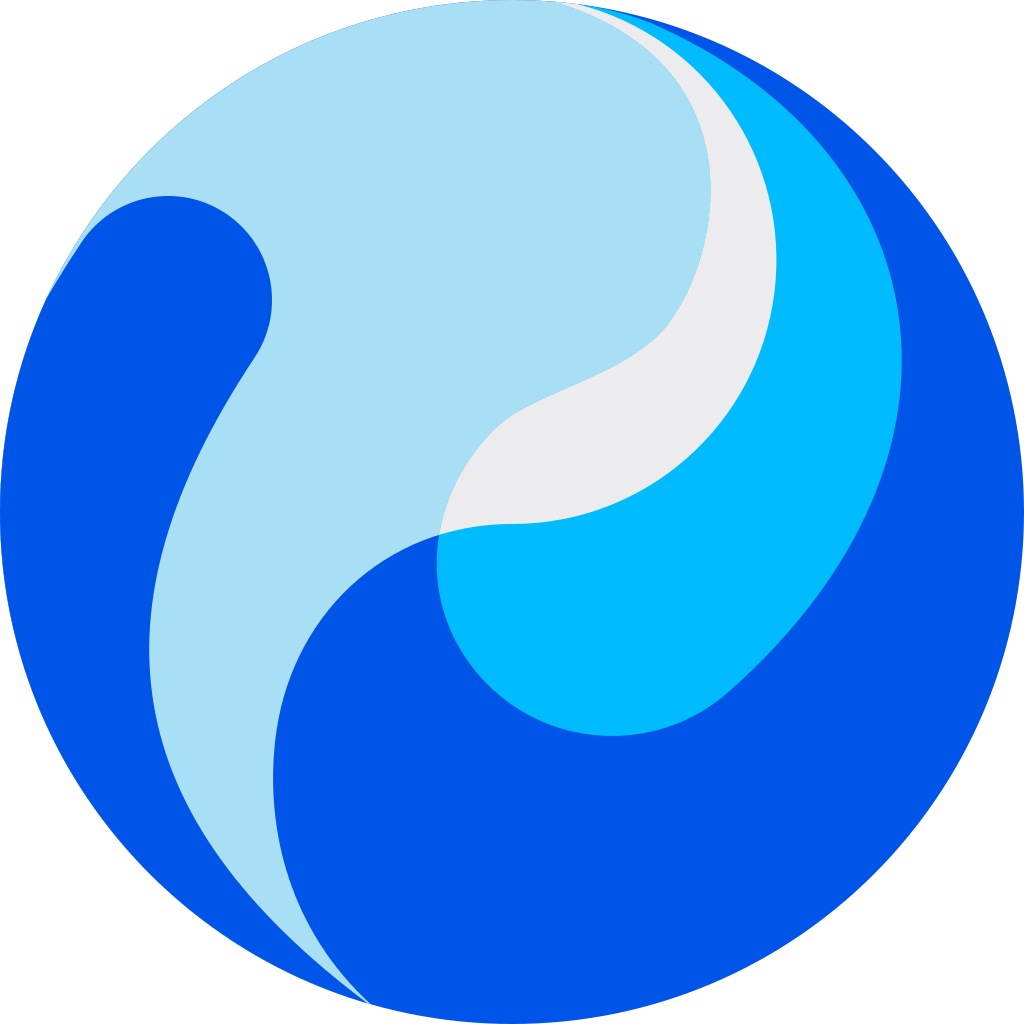}}%
        \hspace{3pt}\textsf{Tencent Hunyuan}%
    }%
    \renewcommand{\headrulewidth}{1pt}%
}

\begin{document}

\maketitle
\fancyhead{}
\thispagestyle{arxivheader}

\begin{figure}[!h]
    \vspace{-0.8cm}
    \centering
    \includegraphics[width=0.9\linewidth]{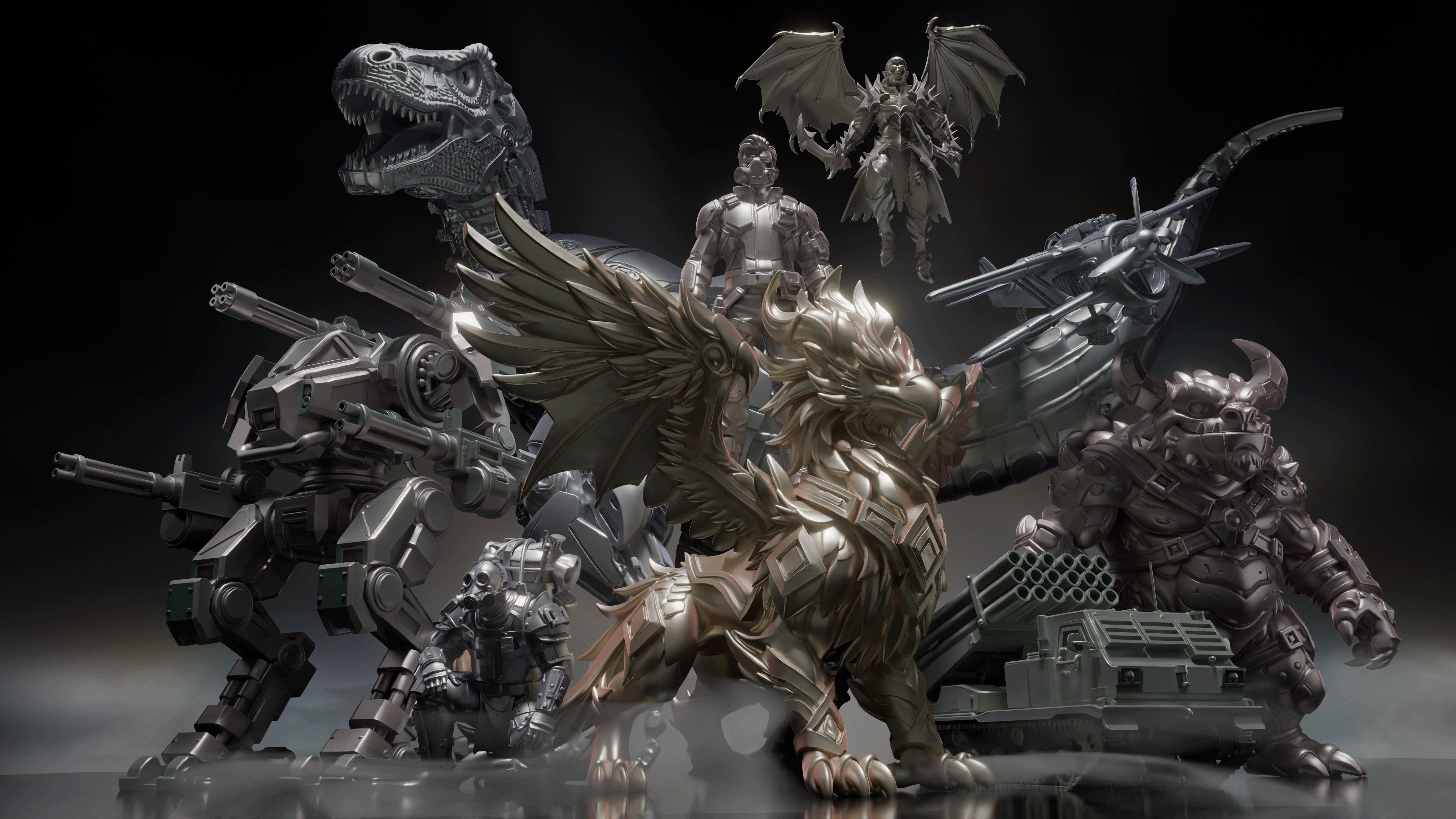}
    \vspace{-0.2cm}
    \caption{We generate diverse 3D assets with plausible global shapes and detailed geometry.}
    \label{fig:teaser}
\end{figure}

\begin{abstract}
Flow matching is central to 3D generation, yet in practice its reinforcement learning (RL) methods are largely adapted from 2D visual generation. Representative DPO-, GRPO-, and NFT-style objectives, when applied to negative trajectories, mainly steer predicted velocities away from the corresponding directions without explicitly specifying a target velocity field toward preferred samples. In 3D generation, constrained by pretrained model capabilities, rollout diversity, and reward-distribution complexity, directly applying these RL methods yields limited gains in geometric quality. We introduce a forward-process RL method \textbf{Dynamic Homing Optimization (DHO)}, which reformulates negative-trajectory optimization as positive-sample attraction-guided dynamic homing. Specifically, Minimum-Cost Attractive Matching (MAM) assigns each negative sample a distinct positive target, and Time-Aware Dynamic Correction (TDC) then redirects its trajectory toward the target using a remaining-time-aware corrective velocity. Building on asynchronous online DHO, we develop \textbf{Flow3D-Pro}, an image-to-3D geometry generation framework. Experiments show that DHO outperforms representative DPO-, GRPO-, and NFT-style objectives in 3D generation, while Flow3D-Pro produces higher-quality 3D geometry than existing mesh generation methods.
\end{abstract}

\vspace{-0.4cm}
\section{Introduction}

3D generation has become a cornerstone of modern content creation, powering high-fidelity asset synthesis~\cite{zhang2024clay,hunyuan3d25,liu2026dreamcharacter13d, Mesh-Pro} for gaming, and increasingly underpinning embodied intelligence. At the heart of this progress lie continuous-time generative models---particularly flow matching and its rectified flow instantiation~\cite{flow_matching,retified_flow}, which learn velocity fields from linear conditional paths between paired prior and data samples and now serve as the backbone of recent 3D generators~\cite{xiang2026trellis2,lai2025lattice,li2026triflow}. Yet 3D generation advances have concentrated mostly on pre-training, while reinforcement learning (RL) post-training remains comparatively underexplored in the flow-based setting. For want of dedicated formulations, current practice~\cite{Flow3D-OPD} transplants RL algorithms developed for 2D visual generation, chiefly DPO-style~\cite{Qwen-Image-TR} (e.g., Flow-DPO \cite{Flow-DPO}), GRPO-style~\cite{MixGRPO,DanceGRPO} (e.g., Flow-GRPO\cite{Flow-GRPO}) and NFT-style~\cite{li2025uniworldv2} (e.g., DiffusionNFT \cite{DiffusionNFT}).

In Figure~\ref{fig:Intro_RL_Comparison}, for the negative-sample branch objective, DPO-style methods implicitly suppress the likelihood of the negative trajectory, and GRPO-style methods reduce the probability of samples with negative advantages. NFT-style methods reverse the gradient through an implicit negative policy, mainly driving the predicted velocity away from the negative-trajectory target. Despite differing in mechanics, these paradigms share one trait in how they treat a negative sample: these objectives mainly repel the velocity directions that produce the corresponding trajectories, but do not explicitly specify a target velocity field that would instead generate preferred samples. They mainly rely on the positive-sample branch objective to guide the velocity field toward generating desirable trajectories. When applied to post-training for 3D generation, our empirical results find that, constrained by the capabilities of the pretrained foundation model, rollout diversity, and the complexity of the reward distribution, these RL methods yield limited gains in geometric quality.

\begin{figure*}[!t]
    \centering
    \includegraphics[width=1.0\textwidth]{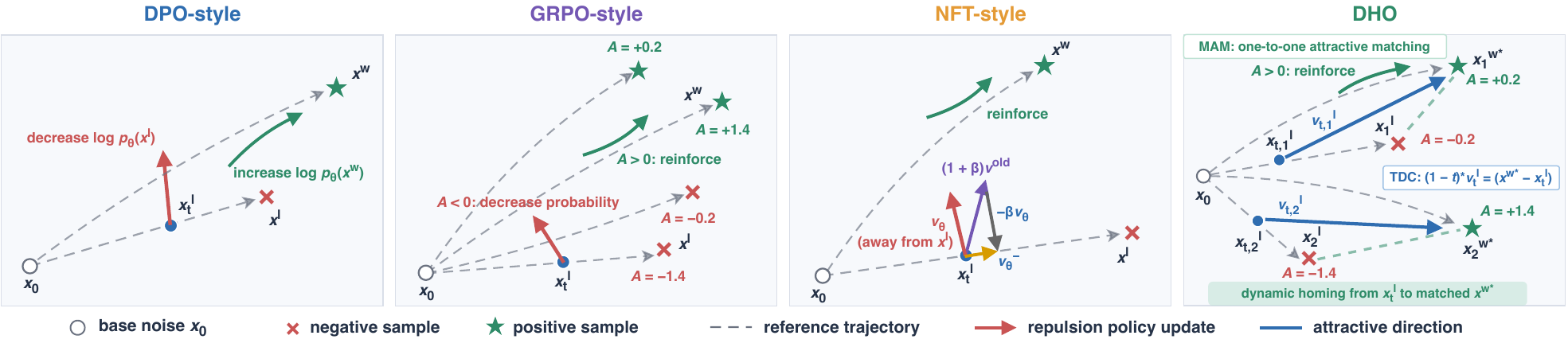}
    \vspace{-0.6cm}
    \caption{Negative-sample branch objectives across flow matching RL methods.}
    \label{fig:Intro_RL_Comparison}
    \vspace{-0.4cm}
\end{figure*}

Inspired by recent positive-sample attraction approaches~\cite{Drifting_Model}, we propose \textbf{Dynamic Homing Optimization (DHO)}, which constructs negative-sample branch objectives with attraction-guided dynamic homing. DHO explicitly specifies both a preferred-sample destination for each negative trajectory and the corrective dynamics required to reach it. Specifically, we introduce Minimum-Cost Attractive Matching (MAM) to formulate preferred-target selection as a batchwise minimum-cost one-to-one assignment, pairing each negative sample with a distinct positive sample. Given the assigned target, Time-Aware Dynamic Correction (TDC) constructs a remaining-time-aware velocity that redirects each intermediate state toward its target. 

DHO performs policy updates on forward-process interpolants constructed from rollout endpoints, is sampler-flexible, and does not require storing entire sampling trajectories. In addition, we develop \textbf{Flow3D-Pro}, an image-to-3D mesh generation framework, which leverages asynchronous online DHO to improve geometric fidelity and structural plausibility. Our main contributions are as follows:
\begin{itemize}
    \item We present Dynamic Homing Optimization (DHO), a forward-process RL method for flow matching that equips the negative-sample branch objectives with positive-sample attraction-guided dynamic homing, and introduce Flow3D-Pro for high-quality image-to-3D mesh generation with asynchronous online DHO.
    \item In DHO, we introduce Minimum-Cost Attractive Matching (MAM) to assign each negative sample a distinct positive target through batchwise minimum-cost one-to-one matching, and Time-Aware Dynamic Correction (TDC) to construct a remaining-time-aware corrective velocity toward that preferred target.
    \item Experiments demonstrate that DHO outperforms representative DPO-, GRPO-, and NFT-style objectives in 3D generation, while Flow3D-Pro achieves higher geometry generation quality than current mesh generation methods.
\end{itemize}
\vspace{-0.2cm}
\section{Related Work}
\subsection{3D Mesh Generation}
Continuous-time generative models underpin modern synthesis: diffusion models established high-fidelity stochastic generation~\cite{DDPG,score_based_song,karras2022edm}, while flow matching~\cite{flow_matching} and its rectified flow instantiation~\cite{retified_flow} learn velocity fields for transport, with related stochastic-interpolant and fast-sampling variants further improving efficiency~\cite{albergo2023stochastic,song2023consistencymodels,InstaFlow,PeRFlow}. In 3D generation, DreamFusion~\cite{poole2022dreamfusion} and Magic3D~\cite{lin2023magic3d} popularized 3D mesh optimization via 2D diffusion distillation, followed by higher-quality variants~\cite{wang2023prolificdreamer,chen2023fantasia3d}. Native 3D generators have progressed from point or implicit representations~\cite{nichol2022point,jun2023shape} to SDF and shape-latent models~\cite{cheng2023sdfusion,zhang2023shape2vecset,zhao2023michelangelo}, as well as textured-mesh and controllable asset pipelines~\cite{gao2022get3d,zhang2024clay,li2024craftsman3d}. Recent rectified-flow 3D systems---including TRELLIS~\cite{xiang2024trellis, xiang2026trellis2}, Hunyuan3D~\cite{hunyuan3d20,hunyuan3d21, hunyuan3d25}, TripoSG~\cite{li2025triposg}, LATTICE~\cite{lai2025lattice} and other mesh flow generators~\cite{sun2026meshflow,li2026triflow}---mainly study pre-training, representation, or sampling, whereas DHO studies RL post-training for 3D rectified-flow geometry generators.

\subsection{Reinforcement Learning for Flow Matching}
RL was adopted early for language-model alignment~\cite{guo2025deepseekr1, rafailov2023direct}, and later extended to visual generation using human feedback~\cite{kirstain2023pickapic} and reward-guided fine-tuning algorithms~\cite{xu2023imagereward, lee2023aligning}. Diffusion reinforcement spans policy-gradient tuning (DDPO~\cite{black2023trainingdiffusionrl}), direct preference optimization (Diffusion-DPO~\cite{DiffusionDPO}), and preference-tuned visual generation variants~\cite{fan2023dpok,clark2023direct}. For flow matching, representative objectives: GRPO-style methods~\cite{Flow-GRPO,Qwen-Image-2.0-RL}, represented by Flow-GRPO~\cite{Flow-GRPO}, treat generation as online RL and penalize low-reward rollouts, with related variants refining stability, reward density or policy granularity~\cite{DanceGRPO,deng2026densegrpo,luo2025gcpo,MixGRPO}. DPO-style adaptations suppress negative trajectories through pairwise preference ratios~\cite{Flow-DPO,Qwen-Image-TR}. NFT-style methods~\cite{li2025uniworldv2}, represented by DiffusionNFT~\cite{DiffusionNFT}, inject negative velocity-space gradients on the forward process. In 3D generation, flow-based RL methods (e.g., Flow3D-OPD~\cite{Flow3D-OPD}) are primarily adapted from these 2D visual generation algorithms. However, these objectives mainly push the velocity away from directions that produce negative trajectories, without explicitly specifying a target velocity field that guides generation toward preferred samples, limiting gains in geometric quality. DHO addresses this by dynamically homing negative samples toward preferred regions. PFM~\cite{kim2024preferencealignmentflowmatching} trains an additional preference-flow model on rejected-to-preferred interpolants, and applies it to outputs from the frozen base model at inference, potentially introducing a source-distribution mismatch. DHO directly updates the original generative base field on noise-to-negative interpolants using matched positive endpoints.

\section{Methodology}

\begin{figure*}[!t]
    \centering
    \includegraphics[width=0.98\textwidth]{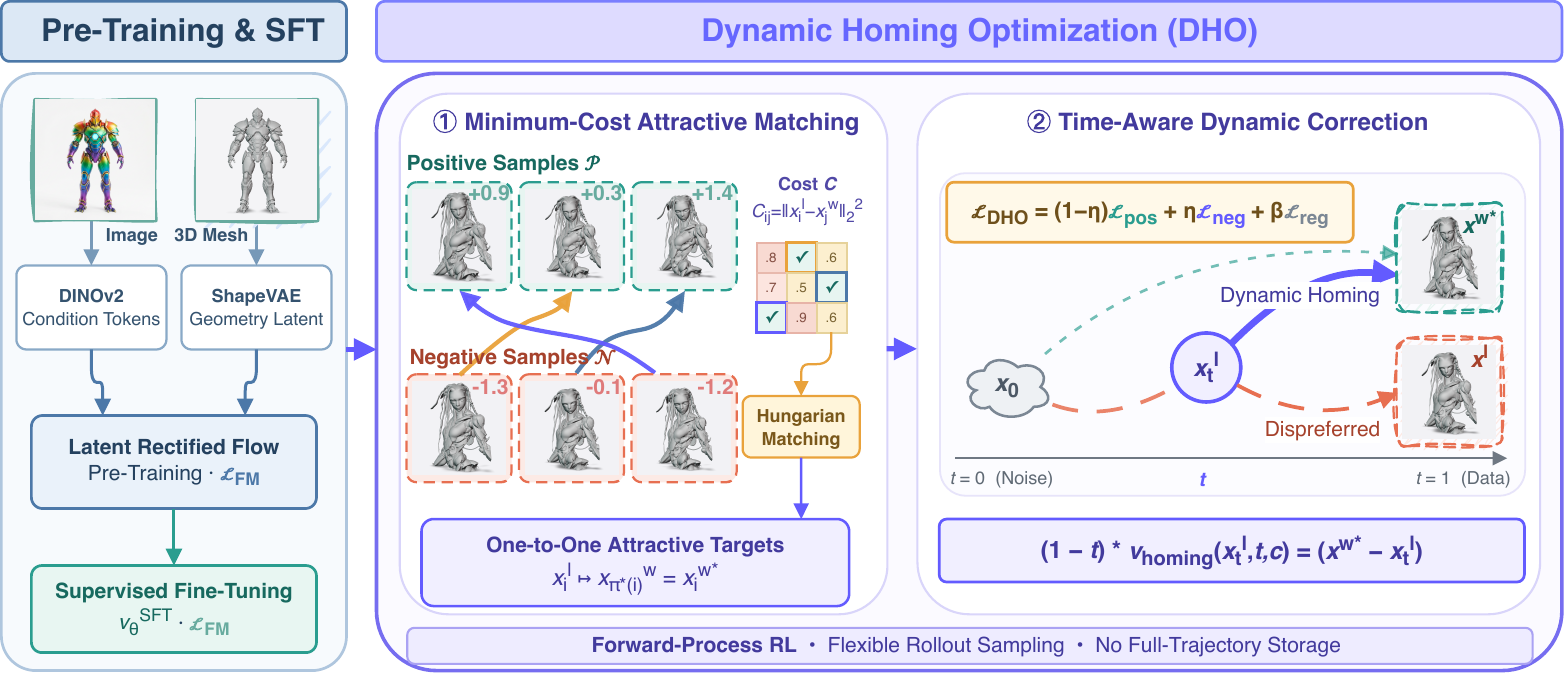}
    \vspace{-0.2cm}
    \caption{\textbf{Method Overview.} After pre-training and supervised fine-tuning (SFT), DHO post-trains the rectified flow model in latent feature space using reward-ranked rollout samples. In DHO, MAM assigns each negative sample an attractive positive-target destination, and TDC then constructs a time-aware homing velocity toward the assigned target. We visualize only the negative-sample branch, as the positive-sample branch uses the same MAM rule and homing-velocity construction.}
    \vspace{-0.2cm}
    \label{fig:dho_overview}
\end{figure*}

\subsection{Rectified Flow as Policy Optimization}
\label{sec:prelim}
We view rectified flow as a continuous-time policy $v_\theta(x_t, t, c)$ that predicts a velocity (action) given state $x_t$, time $t \in [0,1]$, and condition $c$. Starting from a Gaussian prior $x_0 \sim p_0 = \mathcal{N}(0, \mathbf{I})$, generation follows the velocity field $v_\theta$ and returns the terminal state $x_1 \sim p_1$ (target data distribution). The policy is trained with the flow-matching objective on the linear path $x_t = (1-t)\,x_0 + t\,x_1$,
\begin{equation}
    \label{eq:fm}
    \mathcal{L}_{\text{FM}}(\theta) = \mathbb{E}_{t \sim \mathcal{U}(0,1),\, x_0 \sim p_0,\, x_1 \sim p_1}
    \left[ \big\| v_\theta(x_t, t, c) - (x_1 - x_0) \big\|_2^2 \right].
\end{equation}
In preference optimization, each condition $c$ is associated with positive samples $\mathcal{P}$ and negative samples $\mathcal{N}$, and policy optimization reshapes $v_\theta$ so that trajectories terminate on $\mathcal{P}$ rather than $\mathcal{N}$.

\subsection{Dynamic Homing Optimization}
Let $x^w \in \mathcal{P}$ and $x^l \in \mathcal{N}$ denote a positive sample and a negative sample, and let $x_t^l = (1-t)x_0 + t\,x^l$ be an intermediate state along the negative trajectory. We propose \textbf{Dynamic Homing Optimization (DHO)}, an attraction-guided navigation paradigm (as illustrated in Figure~\ref{fig:Intro_RL_Comparison} and Figure~\ref{fig:dho_overview}) with two main components: Minimum-Cost Attractive Matching (MAM), which assigns each negative sample a distinct attractive target within the batch, and Time-Aware Dynamic Correction (TDC), which constructs a remaining-time-aware velocity toward that target. Although our analysis focuses primarily on negative-sample branch objectives, the same target-assignment and velocity-target construction also apply to positive-sample branch objectives. 

\subsubsection{Minimum-Cost Attractive Matching}
\label{sec:assignment}
We formulate attractive target selection as a minimum-cost one-to-one assignment problem within each mini-batch. In latent feature space, given $B$ negative states $\hat{\mathcal{N}} = \{x_i^l\}_{i=1}^{B}$ and $B$ positive states $\hat{\mathcal{P}} = \{x_j^w\}_{j=1}^{B}$, we form the cost matrix $\mathbf{C} \in \mathbb{R}^{B \times B}$,
\begin{equation}    
    \label{eq:cost}
    C_{i,j} = \mathcal{D}(x_i^l, x_j^w) = \big\| x_i^l - x_j^w \big\|_2^2 ,
\end{equation}
and seek the permutation $\boldsymbol{\pi}^*$ that minimizes the total assignment cost:
\begin{equation}
    \label{eq:lsa}
    \begin{split}
    \boldsymbol{\pi}^* &= \arg\min_{\boldsymbol{\pi} \in \Pi_B}
    \sum_{i=1}^{B} \sum_{j=1}^{B} C_{i,j}\, \pi_{i,j}, \\
    \Pi_B &= \Big\{ \boldsymbol{\pi} \in \{0,1\}^{B \times B} : \textstyle\sum_i \pi_{i,j} = \sum_j \pi_{i,j} = 1 \Big\}.
    \end{split}
\end{equation}
Here $\Pi_B$ denotes the set of $B\times B$ permutation matrices. This assignment can be viewed as minibatch optimal transport between the uniform empirical distributions of negative and positive rollout endpoints~\cite{minibatch_ot}. Eq.~\eqref{eq:lsa} is a linear sum assignment problem solved in $\mathcal{O}(B^3)$ by the Hungarian algorithm. For a selected optimal permutation, $\pi^*(i)$ denotes the unique column index $j$ with $\pi^*_{i,j}=1$. The assigned attractive target of a negative sample is
\begin{equation}
    \label{eq:assign}
    x_i^{w*} = x_{\pi^*(i)}^{w}.
\end{equation}
Moreover, the same assignment principle extends to the positive-sample branch objective: applying Eq.~\eqref{eq:lsa} between $\hat{\mathcal{P}}$ and itself admits the identity permutation as an optimal assignment because the diagonal has zero cost, $C_{i,i}=\|x_i^w-x_i^w\|_2^2=0$. Choosing this assignment matches every positive sample to itself, $x_i^{w*}=x_i^{w}$, so both branches are governed by the single assignment rule of Eq.~\eqref{eq:assign}.

Given a rollout group of $G$ samples with rewards $\{r_i\}_{i=1}^{G}$, we compute the group-normalized advantage $A_i=(r_i-\bar{r})/s_r$, where $\bar{r}$ and $s_r$ are the group mean and standard deviation. Samples with $A_i>0$ are designated as positive, whereas those with $A_i<0$ are designated as negative. Zero-advantage samples are omitted. To construct equally sized sets for MAM, let $B$ be the smaller set size. If positive samples are more numerous, we retain the $B$ samples with the highest advantages. If negative samples are more numerous, we randomly select $B$ of them. For each retained positive sample, we define the normalized advantage weight $\omega_i^w=A_i^w/(\frac{1}{B}\sum_{j=1}^{B}A_j^w)$. After MAM, a negative sample inherits the weight of its assigned positive target, $\omega_i^{w*}=\omega_{\pi^*(i)}^w$.

\subsubsection{Time-Aware Dynamic Correction}
\label{sec:corrective}
Given the assigned target $x_i^{w*}$, a naive choice is the static parallel action $x^{w*} - x_0$. This is incorrect for a negative trajectory: the state $x_t^l$ already lies off the straight path to $x^{w*}$, so integrating a constant velocity accumulates a terminal error proportional to $t\,(x^l - x^{w*})$. To account for this residual displacement, the action must depend on both the current state and the remaining time.

We therefore define a time-aware corrective field that points from the current state to the target and rescales by the remaining horizon $1-t$, so that the residual displacement is covered at $t=1$:
\begin{equation}
    \label{eq:correction}
    v_{\text{target}}^{\text{correction}}(x_t^l, t, c) = \frac{x^{w*} - x_t^l}{1 - t}.
\end{equation}
Integrating Eq.~\eqref{eq:correction} directs the state to the endpoint $x^{w*}$.

The factor $1/(1-t)$ in Eq.~\eqref{eq:correction} diverges as $t \to 1$. To obtain a numerically stable regression target for the negative branch, we adopt the time-dependent weight $w_{\text{correction}}(t) = (1-t)^2$, which cancels the singularity: weighting the squared residual of ($v_\theta - v_{\text{target}}^{\text{correction}}$) by $w_{\text{correction}}(t)$ is equivalent to matching the rescaled velocity $(1-t)\,v_\theta$ against a bounded target. Substituting Eq.~\eqref{eq:correction} and $x_t^l = t\,x^l + (1-t)x_0$ gives
\begin{equation}
    \label{eq:correction-simplify}
    (1-t)\, v_{\text{target}}^{\text{correction}}(x_t^l, t, c) = x^{w*} - x_t^l = (x^{w*} - x_0) - t\,(x^l - x_0),
\end{equation}
which is smooth on all of $[0,1]$. Weighting the correction by the normalized advantage of its assigned positive target yields the policy optimization objective for the negative-sample branch:
\begin{equation}
\label{eq:neg-loss}
\begin{split}
    \mathcal{L}_{\text{neg}}(\theta) = \mathbb{E}_{t, x_0, x^l} \Big[ \omega^{w*} \big\| (1-t)\,v_\theta(x_t^l, t, c) - \big( (x^{w*} - x_0) - t\,(x^l - x_0) \big) \big\|_2^2 \Big].
\end{split}
\end{equation}
Eq.~\eqref{eq:neg-loss} gives stronger correction weights toward higher-advantage positive targets while keeping the signal well-conditioned across the whole time horizon. Since $\omega^{w*}>0$ only rescales the regression loss, it preserves the homing target of Eq.~\eqref{eq:correction-simplify}.

\label{sec:objective}
The positive-sample branch shares the MAM rule and homing-velocity construction with the negative-sample branch. Self-matching the positive set returns the identity, so a positive sample is assigned to itself, $x^{w*}=x^{w}$. Substituting this target into the time-aware field of Eq.~\eqref{eq:correction} and using $x_t^w = (1-t)x_0 + t\,x^w$ collapses the correction to the standard rectified-flow action:
\begin{equation}
    \label{eq:pos-degenerate}
    v_{\text{target}}^{\text{correction}}(x_t^w, t, c) = \frac{x^{w} - x_t^w}{1-t}
    = \frac{(1-t)\,(x^{w} - x_0)}{1-t} = x^{w} - x_0 .
\end{equation}
Since the self-matched velocity target has no endpoint singularity, we retain the rectified-flow action with positive-advantage weighting and no additional temporal weighting:
\begin{equation}
    \label{eq:pos-loss}
    \mathcal{L}_{\text{pos}}(\theta) = \mathbb{E}_{t, x_0, x^w}
    \left[ \omega^w \big\| v_\theta(x_t^w, t, c) - (x^w - x_0) \big\|_2^2 \right].
\end{equation}

\subsubsection{Full Optimization Objective}
\label{sec:reg}
To constrain policy deviation during post-training, we regularize the learned velocity field toward a frozen reference policy $v_{\theta_{\text{ref}}}$. Specifically, we use the squared velocity deviation as a likelihood-free reference-policy regularizer:
\begin{equation}
    \label{eq:reg}
    \mathcal{L}_{\text{reg}}(\theta) = \mathbb{E}_{t,\, x_0,\, x \in \hat{\mathcal{P}} \cup \hat{\mathcal{N}}}
    \left[ \big\| v_\theta(x_t, t, c) - v_{\theta_{\text{ref}}}(x_t, t, c) \big\|_2^2 \right].
\end{equation}
Combining Eqs.~\eqref{eq:neg-loss},~\eqref{eq:pos-loss}, and~\eqref{eq:reg}, the full DHO objective jointly anchors positive trajectories, dynamically homes negative ones, and regularizes toward the reference policy:
\begin{equation}
    \label{eq:total}
    \mathcal{L}_{\text{DHO}}(\theta) = \eta \, \mathcal{L}_{\text{neg}}(\theta) 
    + (1-\eta) \, \mathcal{L}_{\text{pos}}(\theta)
    + \beta \, \mathcal{L}_{\text{reg}}(\theta),
\end{equation}
where $\eta \in (0,1)$, $\beta > 0$ control the trade off between homing strength and reference regularization. 

\subsubsection{Endpoint Behavior Analysis}
DHO explicitly defines a homing field toward a positive target. We analyze the endpoint behavior of DHO's homing field from two perspectives: the converged endpoint and the training dynamics. Since optimizing $x^w$ is standard flow matching, we mainly focus on the negative branch $x^l$. 

\begin{proposition}[]
\label{prop:mode_transformation}
Let $x_0 \sim p_0$ be a base noise sample that deterministically maps to a negative sample $x^l$ under the pre-trained reference policy $v_{\theta_{\text{ref}}}$. As the policy $v_\theta$ converges to the optimal dynamic homing field $v^*(x_t, t, c) = \frac{x^{w*} - x_t}{1-t}$ via training, the generative mapping $F_\theta: x_0 \mapsto x_1$ shifts its terminal endpoint such that $F_{\theta^*}(x_0) = x^{w*}$.
\end{proposition}
This characterizes the converged endpoint: as the policy converges to the optimal homing field, the deterministic map $F_\theta$ is redirected so that the specific prior sample $x_0$ that previously generated $x^l$ is routed to its assigned positive mode $x^{w*}$.

\begin{proposition}[]
\label{prop:training_dynamics}
For a fixed $x^l$ and its assigned $x^{w*}$, assume $ v_\theta^{\text{init}}(x, t) = \frac{x^l - x}{1-t}$ for $t \in [0,1)$. Let the selected policy snapshots satisfy the idealized velocity update $v_\theta^{\tau+1}(x, t) = v_\theta^\tau(x, t) + \alpha \big( \frac{x^{w*} - x}{1-t} - v_\theta^\tau(x, t) \big)$ pointwise over $(x,t)$, where $\tau \ge 0$ indexes successive policy snapshots rather than necessarily
every raw SGD step and $\alpha \in (0,1)$ is an effective contraction factor between consecutive snapshots. Then the terminal generation is $x_1(\tau) = \big(1 - (1-\alpha)^\tau\big) x^{w*} + (1-\alpha)^\tau x^l$, and hence $\lim_{\tau \to \infty} x_1(\tau) = x^{w*}$.
\end{proposition}

\begin{wrapfigure}{r}{0.50\textwidth}
    \centering
    \vspace{-0.4cm}
    \includegraphics[width=\linewidth]{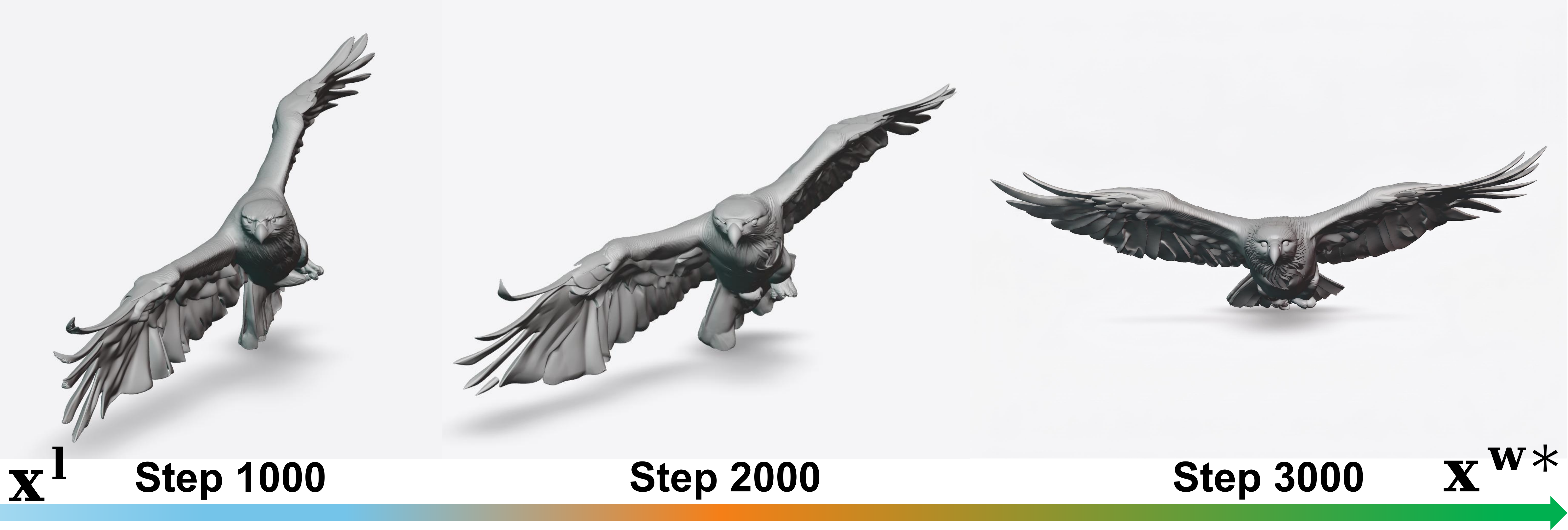}
    \caption{The generated owl gradually becomes canonically aligned during DHO training.}
    \vspace{-0.5cm}
    \label{fig:convergence_training_process}
\end{wrapfigure}
This is an idealized field-space analysis of how the terminal endpoint of a particular negative generation evolves when selected policy snapshots contract pointwise toward its homing field. $x_1(\tau)$ interpolates between $x^l$ and $x^{w*}$ with weight $(1-\alpha)^\tau$, closing the residual gap at a geometric rate. Figure~\ref{fig:convergence_training_process} illustrates this behavior qualitatively.

Moreover, MAM is recomputed across mini-batches and may reassign the positive target. These propositions therefore apply conditionally to a fixed assignment and piecewise to intervals with different assignments: each interval supplies a constructive correction toward its current positive target. If the assignments eventually stabilize at $x^{w*}$, the stated recurrence predicts convergence to that target. As training progresses, the velocity field within each interval moves toward a more positive direction. Detailed proofs for these propositions are in the appendix.

\subsubsection{Classifier-Free Guidance for DHO}
\label{sec:cfg}
DHO is compatible with Classifier-Free Guidance (CFG), whose guidance scale balances exploration diversity and generation quality. During pre-training and SFT, the conditioning embedding $\mathbf{c}$ is randomly replaced with a learnable null embedding $\mathbf{c}_{\varnothing}$ with a fixed probability, enabling the model to learn both conditional and unconditional velocity fields. Given a guidance scale $s \ge 1$ ($s=1$ corresponding to the CFG-free setting), the guided velocity is
\begin{equation}
    \label{eq:cfg-infer}
    \hat{v}_\theta(x_t,t,\mathbf{c}) = v_\theta(x_t,t,\mathbf{c}_{\varnothing})
    + s\big(v_\theta(x_t,t,\mathbf{c}) - v_\theta(x_t,t,\mathbf{c}_{\varnothing})\big).
\end{equation}
We consistently apply Eq.~\eqref{eq:cfg-infer} to rollout generation, policy updates, and inference using the same guidance scale $s$. This consistency ensures that policy optimization is performed under the same guided dynamics used for exploration and exploitation. The DHO optimization objective remains unchanged except that $v_\theta$ is replaced by $\hat{v}_\theta$. 

\subsection{Flow3D-Pro}
Building on DHO, we develop \textbf{Flow3D-Pro}, a two-stage framework for image-to-3D mesh generation. Directly recovering detailed 3D geometry from a 2D reference image remains challenging. Following prevailing multi-stage paradigms~\cite{lai2025lattice,xiang2026trellis2,ROAR-3D}, Flow3D-Pro decomposes generation into a lifting stage (Stage I) and a refinement stage (Stage II). Stage I lifts the 2D reference image into 3D space to establish the global geometry and spatial arrangement of the mesh, whereas Stage II refines the resulting mesh with finer geometric details. The lifting stage is built upon Hunyuan3D 2.1~\cite{hunyuan3d21}, while the refinement stage is derived from LATTICE~\cite{lai2025lattice}.

Stage I establishes the global geometry and structural initialization---including the overall shape, pose, and spatial arrangement---that constrain subsequent refinement. We therefore focus post-training on Stage I, whose high output diversity during 2D-to-3D lifting provides a broader exploration space for RL to discover improved policies. Specifically, Stage I is initialized from pretrained weights and then optimized in two sequential post-training phases: SFT, followed by DHO. Stage II remains frozen at its pretrained weights and is used only at inference time to refine the geometry produced by Stage I, primarily by enhancing local surface details.

\subsubsection{Asynchronous Online DHO}
Flow3D-Pro rolls out $G$ meshes conditioned on each reference image. For each mesh, we render six views and provide them together with the reference image to Gemini 3.1 Pro, which serves as the reward model and assigns a scalar reward based on generation quality. Gemini evaluates each mesh along three dimensions: reconstruction fidelity, geometric plausibility, and orientation alignment. Reconstruction fidelity assesses the consistency of silhouette proportions and surface appearance with the reference image. Geometric plausibility considers the validity of the overall shape and its consistency across views. Orientation alignment evaluates whether the generated mesh is aligned with the coordinate axes, which helps the learning of symmetric structures. By jointly considering these three dimensions, Gemini assigns each mesh a reward score from 1 to 5 (higher is better).

In practice, reward evaluation by Gemini introduces substantial and highly variable latency across meshes. To improve training efficiency, we adopt an asynchronous online DHO framework following~\cite{Mesh-Pro}, decoupling mesh rollout and reward evaluation from policy optimization. We allocate $N$ GPUs to rollout and $M$ GPUs to policy updates. Each rollout worker holds a fixed snapshot of the latest released policy weight while generating a complete mesh group. Once Gemini evaluation is complete, the scored group is inserted into a rollout replay buffer. In parallel, the trainer workers continuously sample evaluated groups from the buffer and update the current policy using the DHO objective. The trainer periodically publishes a new policy snapshot, which rollout workers fetch before starting subsequent rollouts. Data generated by outdated policy versions are discarded to limit policy staleness.

We balance policy synchronization and replay-buffer utilization through an update-to-buffer ratio. Let $K$ denote the number of policy-update steps between two consecutive policy releases, $b$ the per-GPU training batch size, and $C$ the buffer capacity measured in training samples. We constrain
\begin{equation}
    \label{eq:async-update-buffer-ratio}
    \lambda_{\min} \leq \frac{K M b}{C} \leq \lambda_{\max},
\end{equation}
where $KMb$ is the number of training samples consumed between policy releases. This constraint balances insufficient learning from newly collected rewards against excessive reuse of stale rollout data, allowing rollout, reward evaluation, and DHO policy updates to proceed concurrently.

\section{Experiments}

\subsection{Dataset}
We collect artist-created and AI-generated meshes and process them through several steps: (a) data filtering, (b) watertight conversion, (c) point-cloud sampling and signed distance field (SDF) extraction, and (d) manual curation. Human annotators select high-quality meshes by considering global shape quality and local geometric details, resulting in 5K meshes for SFT. For DHO, we collect 1K diverse reference images, covering 3D and 2D cartoon styles as well as photorealistic subjects. The collection spans humans, animals, vehicles, architecture, plants and other object categories.

We assess generalization on two evaluation sets disjoint from training. The first, \textit{GenMesh-Test}, is our curated set of 200 challenging reference images (comparable in scale to the evaluation set of~\cite{lai2025lattice}) covering diverse object categories, visual styles, and levels of geometric complexity. The second, \textit{LATTICE-Bench}, is the evaluation set adopted by LATTICE~\cite{lai2025lattice}. 

\subsection{Implementation Details}
DHO is trained for 5K optimization steps with AdamW at a learning rate of $2.5\times10^{-6}$. Training is performed in bfloat16 with a per-GPU batch size of 2. We set $\eta=0.1$, $\beta=1.0$, the rollout group size to $G=8$, and the CFG scale to $3.0$. For asynchronous optimization, we allocate $N=28$ GPUs to rollout and $M=4$ GPUs to policy updates, with $\lambda_{\min}=4$, $\lambda_{\max}=8$, and $K=16$. Stage II is kept frozen during inference. Full training and inference configurations are in the appendix.

\subsection{Evaluation Metrics}
We evaluate the image-to-geometry generation through various metrics including ULIP~\cite{ULIP,ULIP-2}, Uni3D~\cite{Uni3D}, following~\cite{lai2025lattice,xiang2026trellis2}, alongside a comprehensive User Study (US) to assess perceptual quality. Human experts rank the generated meshes of U methods from best (U - 1) to worst (0), revealing relative preference relationships. We prioritize the User Study as our primary subjective metric for mesh generation.

\subsection{Comparison of Reinforcement Learning Objectives}
To facilitate a controlled comparison of RL objectives for Stage I, we evaluate the base model, SFT, DHO, and representative RL methods: GRPO-style, DPO-style and NFT-style. All RL variants start from the same SFT weight, share the reward model, and are evaluated on GenMesh-Test. Every Stage I output is processed by the same frozen Stage II refiner. We repeat each experiment with three different random seeds and report the mean results. Implementation details are in the appendix.

\begin{table*}[!t]
  \centering
  \caption{Comparison of RL objectives for Stage I, evaluated before and after the Stage II refiner.}
  \label{tab:main_baselines}
  \resizebox{0.95\linewidth}{!}{
  \begin{tabular}{l|ccc|ccc} 
    \toprule
    \multirow{2}{*}{Method} & \multicolumn{3}{c|}{Stage I} & \multicolumn{3}{c}{Stage II} \\
    \cmidrule(lr){2-4} \cmidrule(lr){5-7}
    & ULIP $\uparrow$ & Uni3D $\uparrow$ & User Study $\uparrow$ & ULIP $\uparrow$ & Uni3D $\uparrow$ & User Study $\uparrow$ \\
    \midrule
    Base       & 0.143 & 0.347 & 0.60 & 0.143 & 0.350 & 0.71 \\
    SFT        & 0.144 \textcolor{blue}{(+0.70\%)} & 0.351 \textcolor{blue}{(+1.15\%)} & 1.93 & 0.143 \textcolor{blue}{(+0.00\%)} & 0.352 \textcolor{blue}{(+0.57\%)} & 1.91 \\
    GRPO-style & 0.147 \textcolor{blue}{(+2.80\%)} & 0.357 \textcolor{blue}{(+2.88\%)} & 3.02 & 0.147 \textcolor{blue}{(+2.80\%)} & 0.358 \textcolor{blue}{(+2.29\%)} & 3.00 \\
    DPO-style  & 0.146 \textcolor{blue}{(+2.10\%)} & 0.357 \textcolor{blue}{(+2.88\%)} & 2.36 & 0.146 \textcolor{blue}{(+2.10\%)} & 0.358 \textcolor{blue}{(+2.29\%)} & 2.32 \\
    NFT-style  & 0.146 \textcolor{blue}{(+2.10\%)} & 0.355 \textcolor{blue}{(+2.31\%)} & 3.07 & 0.147 \textcolor{blue}{(+2.80\%)} & 0.357 \textcolor{blue}{(+2.00\%)} & 2.80 \\
    DHO        & \textbf{0.149} \textbf{\textcolor{blue}{(+4.20\%)}} & \textbf{0.359} \textbf{\textcolor{blue}{(+3.46\%)}} & \textbf{4.02} & \textbf{0.150} \textbf{\textcolor{blue}{(+4.90\%)}} & \textbf{0.360} \textbf{\textcolor{blue}{(+2.86\%)}} & \textbf{4.26} \\
    \bottomrule
  \end{tabular}
  }
\end{table*}

\begin{figure*}[!t]
    \centering
    \includegraphics[width=1.0\textwidth]{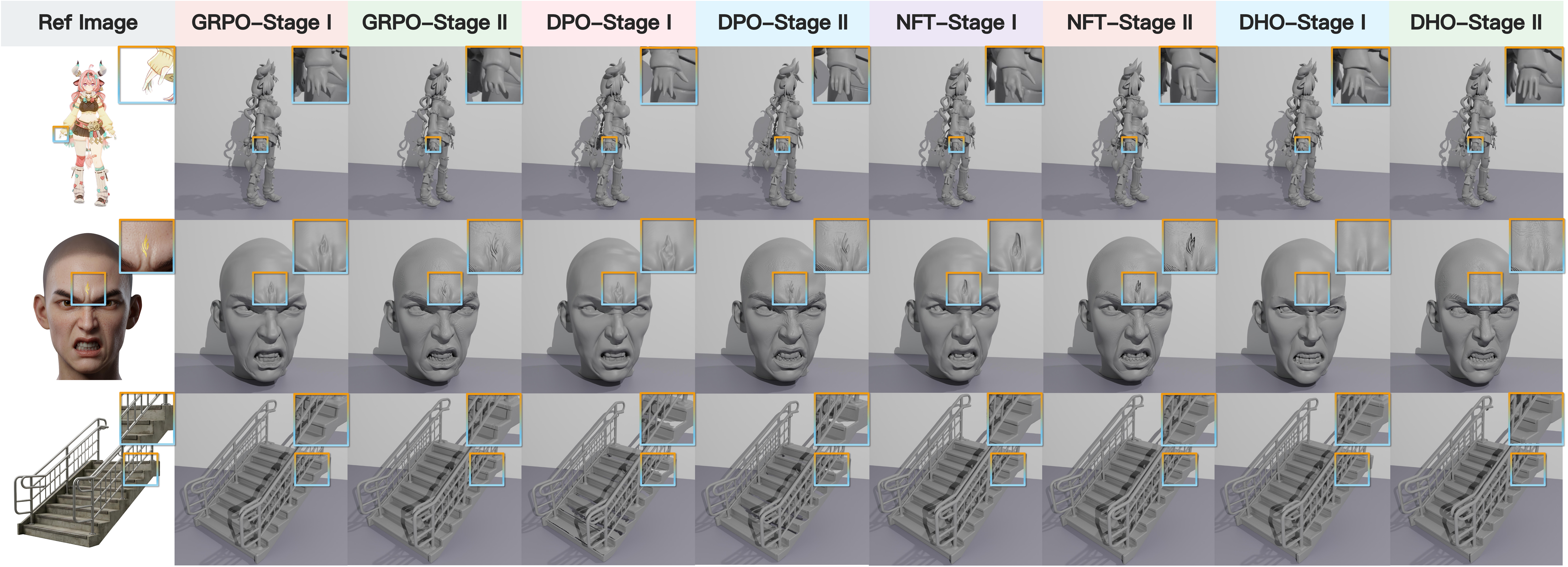}
    \caption{Qualitative comparison of RL objectives. To facilitate visualization, all generated meshes are rotated to a common orientation. The same convention is used in subsequent visualizations.}
    \label{fig:RL_Comparison}
\end{figure*}

Table~\ref{tab:main_baselines} shows that DHO obtains higher scores than the other RL methods, especially on the User Study, our primary metric, demonstrating the advantage of positive-sample attraction-guided dynamic homing. Moreover, the improved Stage I outputs provide a more favorable basis for Stage II refinement, indicating that better Stage I geometry carries over to higher final mesh quality.

Figure~\ref{fig:RL_Comparison} further illustrates qualitative differences among RL objectives. GRPO-style, DPO-style, and NFT-style produce implausible geometry (four fingers), convert image texture into unintended geometry (a flame-shaped forehead), or damage an otherwise intact structure (a broken staircase). These defects arise at Stage I and remain after Stage II refinement, suggesting that the quality of the lifting-stage geometry affects the refined output. In contrast, DHO achieves better geometric generation by attraction-guided navigation.

\begin{figure*}[!t]
    \centering
    \includegraphics[width=1.0\textwidth]{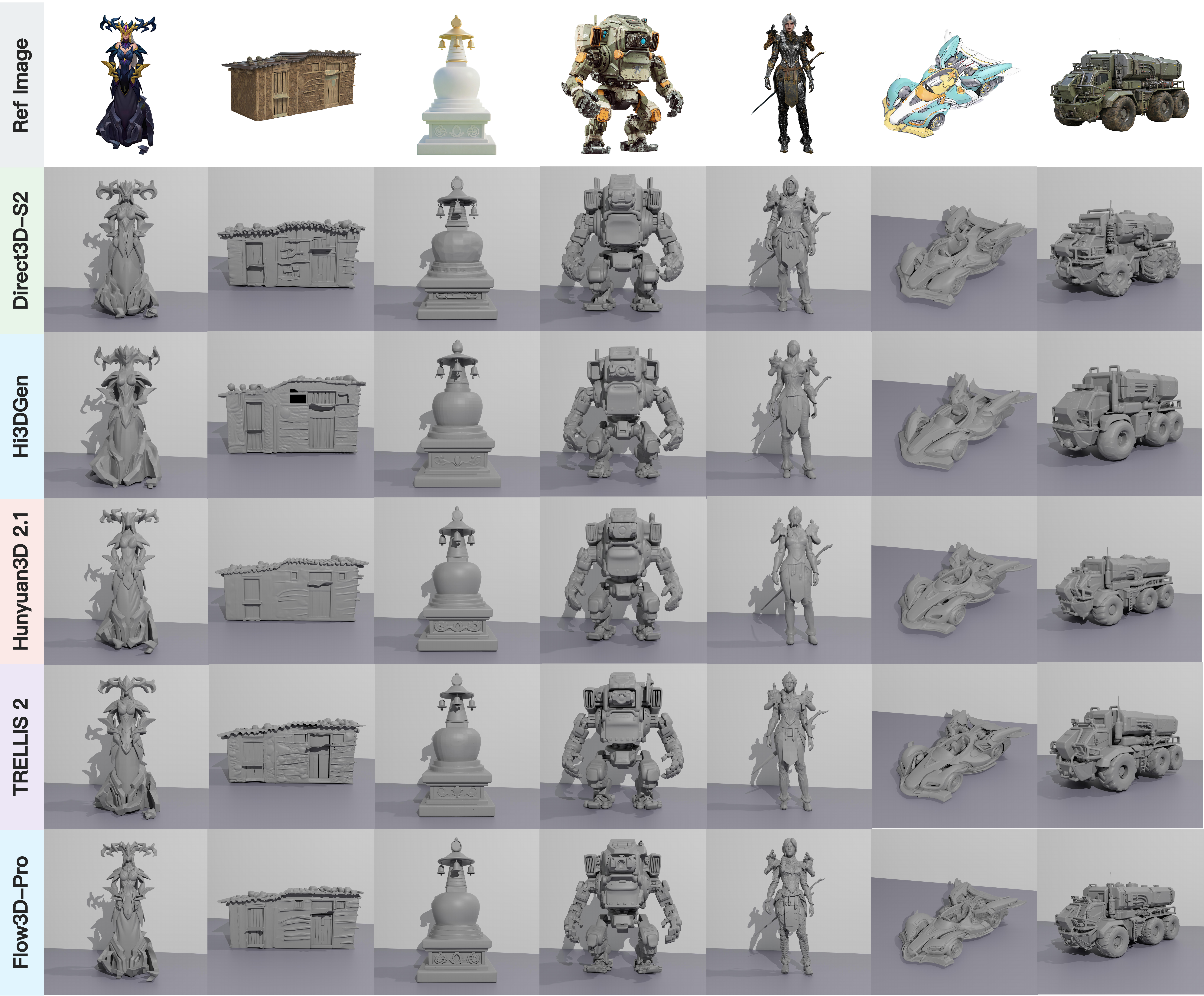}
    \caption{Visual comparison with representative 3D generation methods.}
    \vspace{-0.4cm}
    \label{fig:Comparison_Experiments}
\end{figure*}

\subsection{Comparison with Existing 3D Generation Methods}
\begin{wraptable}{r}{0.56\textwidth}
  \vspace{-0.4cm}
  \centering
  \caption{Quantitative evaluation on 3D mesh generation.}
  \label{tab:performance}
  \resizebox{\linewidth}{!}{
  \begin{tabular}{lccc}
    \toprule
    Method & ULIP $\uparrow$ & Uni3D $\uparrow$ & User Study $\uparrow$ \\
    \midrule
    Michelangelo~\cite{zhao2023michelangelo} & 0.135 & 0.290 & 0.18\\
    Craftsman 1.5~\cite{li2024craftsman3d} & 0.146 & 0.325 & 1.32 \\
    Hi3DGen~\cite{Hi3dgen} & 0.144 & 0.335 & 3.77 \\
    Direct3D-s2~\cite{Direct3D-S2} & 0.141 & 0.345 & 3.27 \\
    Hunyuan3D 2.1~\cite{hunyuan3d21} & 0.143 & 0.347 & 4.14 \\
    Trellis~\cite{xiang2024trellis} & 0.145 & 0.351 & 4.27 \\
    Flow3D-OPD~\cite{Flow3D-OPD} & 0.146 & 0.354 & 7.04 \\
    Trellis2~\cite{xiang2026trellis2} & 0.145 & 0.353 & 6.73 \\
    \midrule
    Flow3D-Pro-Base & 0.143 & 0.350 & 5.64 \\
    Flow3D-Pro & \textbf{0.150} & \textbf{0.360} & \textbf{8.64} \\
    \bottomrule
  \end{tabular}
  }
  \vspace{-0.4cm}
\end{wraptable}

We next compare Flow3D-Pro with representative image-to-3D mesh generators on GenMesh-Test. Flow3D-Pro-Base denotes the Flow3D-Pro variant that uses the pretrained Stage I model without post-training, while retaining the same frozen Stage II refiner.

As shown in Table~\ref{tab:performance}, after Stage I post-training, Flow3D-Pro achieves the highest geometry generation performance among the compared methods on all three metrics, reaching $0.150$ in ULIP, $0.360$ in Uni3D, and $8.64$ in User Study. In Figure~\ref{fig:Comparison_Experiments}, Flow3D-Pro achieves better reference-image adherence, more plausible geometry, and finer geometric details. These experiments demonstrate the role of DHO in improving geometric quality. The appendix provides additional component ablations, evaluations on LATTICE-Bench, and more qualitative results.


\section{Conclusion}
We present Dynamic Homing Optimization (DHO), a forward-process RL method for flow matching that equips negative-sample branch optimization with attraction-guided dynamic homing. DHO combines MAM to assign preferred positive targets with TDC to construct homing velocities toward them. We instantiate DHO in Flow3D-Pro using asynchronous online post-training. Experiments show that DHO outperforms representative DPO-, GRPO-, and NFT-style objectives in mesh generation, while Flow3D-Pro achieves superior image-to-3D geometry quality over existing methods.

\clearpage

\bibliographystyle{iclr2027_conference}
\bibliography{main}

\clearpage
\appendix

\startcontents[supplementary]
\section*{Appendix Contents}
\printcontents[supplementary]{}{1}{\setcounter{tocdepth}{2}}

\section{Methodology}
\subsection{Proofs for Endpoint Behavior Analysis Propositions}

We provide proofs for the endpoint behavior analysis propositions.

\noindent \textbf{Proposition 1.} \textit{Let $x_0 \sim p_0$ be a base noise sample that deterministically maps to a negative sample $x^l$ under the pre-trained reference policy $v_{\theta_{\text{ref}}}$. As the policy $v_\theta$ converges to the optimal dynamic homing field $v^*(x_t, t, c) = \frac{x^{w*} - x_t}{1-t}$ via training, the generative mapping $F_\theta: x_0 \mapsto x_1$ shifts its terminal endpoint such that $F_{\theta^*}(x_0) = x^{w*}$.}

\begin{proof}
After DHO training, we use ordinary differential equation (ODE) sampling for inference. The generation process is an Initial Value Problem (IVP) governed by the ODE:
\begin{equation}
    \frac{dx_t}{dt} = v^*(x_t, t, c) = \frac{x^{w*} - x_t}{1-t}, 
\end{equation}
with initial condition $x(0) = x_0$.
This is a first-order linear ordinary differential equation, which can be rewritten as:
\begin{equation}
    \frac{dx_t}{dt} + \frac{1}{1-t} x_t = \frac{x^{w*}}{1-t}
\end{equation}
We solve this ODE by multiplying both sides by the integrating factor $\mu(t) = \exp\left(\int \frac{1}{1-t} dt\right) = \frac{1}{1-t}$:
\begin{equation}
    \begin{split}
    \frac{1}{1-t}\frac{dx_t}{dt} + \frac{1}{(1-t)^2} x_t = \frac{x^{w*}}{(1-t)^2} \\
    \implies \frac{d}{dt}\left( \frac{x_t}{1-t} \right) = \frac{x^{w*}}{(1-t)^2}
    \end{split}
\end{equation}
Integrating both sides from $0$ to $t$ gives
\begin{equation}
    \begin{split}
    \frac{x_t}{1-t} - \frac{x_0}{1} &= \int_{0}^{t} \frac{x^{w*}}{(1-s)^2} ds \\
    &= x^{w*} \left[ \frac{1}{1-s} \right]_{0}^{t} \\
    &= x^{w*} \left( \frac{1}{1-t} - 1 \right)
    \end{split}
\end{equation}
Multiplying the entire equation by $(1-t)$ isolates $x_t$:
\begin{equation}
    x_t - x_0(1-t) = x^{w*} \left( 1 - (1-t) \right) = t x^{w*}
\end{equation}
Rearranging for $x_t$ yields the linear interpolation path toward the assigned target:
\begin{equation}
    x_t = t x^{w*} + (1-t) x_0
\end{equation}
Taking the limit as $t \to 1$ gives
\begin{equation}
    \lim_{t \to 1} x_t = 1 \cdot x^{w*} + 0 \cdot x_0 = x^{w*}
\end{equation}
This completes the proof. By minimizing the dynamic homing objective, the base noise $x_0$ that originally formulated the negative mode $x^l$ is successfully rewired to map to the positive mode $x^{w*}$.
\end{proof}

\noindent \textbf{Proposition 2.} \textit{For a fixed $x^l$ and its assigned $x^{w*}$, assume $v_\theta^{\text{init}}(x, t) = \frac{x^l - x}{1-t}$ for $t \in [0,1)$. Let the selected policy snapshots satisfy the idealized velocity update $v_\theta^{\tau+1}(x, t) = v_\theta^\tau(x, t) + \alpha \big( \frac{x^{w*} - x}{1-t} - v_\theta^\tau(x, t) \big)$ pointwise over $(x,t)$, where $\tau \ge 0$ indexes successive policy snapshots rather than necessarily
every raw SGD step and $\alpha \in (0,1)$ is an effective contraction factor between consecutive snapshots. Then the terminal generation is $x_1(\tau) = \big(1 - (1-\alpha)^\tau\big) x^{w*} + (1-\alpha)^\tau x^l$, and hence $\lim_{\tau \to \infty} x_1(\tau) = x^{w*}$.}

\begin{proof}
Let the initial policy toward the negative sample be
\begin{equation}
    v_\theta^{\text{init}}(x, t) = \frac{x^l - x}{1-t}.
\end{equation}
For the fixed conditional target $x^{w*}$, the time-weighted negative objective $\mathcal{L}_{\text{neg}}$ and the corresponding flow-matching regression share the minimizer $v^{*}=\frac{x^{w*}-x}{1-t}$. We analyze selected policy snapshots under the idealized assumption that their velocity values contract pointwise toward this minimizer by an effective factor $\alpha$. $\alpha \in (0,1)$ is an effective contraction factor between consecutive snapshots.
\begin{equation}
    v_\theta^{\tau+1}(x, t) = v_\theta^\tau(x, t) + \alpha \big( \frac{x^{w*} - x}{1-t} - v_\theta^\tau(x, t) \big).
\end{equation}
This recurrence is a field-space model between the selected snapshots, not an assertion that every raw Stochastic Gradient Descent (SGD) step follows the same update in the parameter space of a finite neural network. Subtracting $\frac{x^{w*} - x}{1-t}$ from both sides isolates the tracking residual $E^\tau(x, t) = v_\theta^\tau(x, t) - \frac{x^{w*} - x}{1-t}$:
\begin{equation}
    E^{\tau+1}(x, t) = E^\tau(x, t) - \alpha\, E^\tau(x, t) = (1 - \alpha) E^\tau(x, t).
\end{equation}
Unrolling this geometric recurrence from $\tau=0$ gives 
\begin{equation}
    E^\tau = (1-\alpha)^\tau E^0,
\end{equation}
hence
\begin{equation}
    \label{eq:p3-vtau}
    v_\theta^\tau(x, t) = \frac{x^{w*} - x}{1-t} + (1 - \alpha)^\tau \left( v_\theta^{\text{init}}(x, t) - \frac{x^{w*} - x}{1-t} \right).
\end{equation}
We write $\beta_\tau = (1-\alpha)^\tau$. Since $\alpha \in (0,1)$, we have $\beta_\tau \in (0,1]$ with $\beta_0 = 1$ and $\beta_{\tau\to \infty} \to 0$. After the initial velocity $v_\theta^{\text{init}}(x, t) = \frac{x^l - x}{1-t}$ is substituted into~\eqref{eq:p3-vtau}, the two terms share the common gain $\frac{1}{1-t}$ and combine:
\begin{equation}
    \label{eq:p3-combined}
    v_\theta^\tau(x, t)
    = \frac{(1-\beta_\tau)(x^{w*} - x) + \beta_\tau (x^l - x)}{1-t}
    = \frac{m_\tau - x}{1-t},
\end{equation}
where $m_\tau := (1-\beta_\tau)\, x^{w*} + \beta_\tau\, x^l$. The iterate is thus again a well-defined homing field, now targeting the interpolated anchor $m_\tau$ rather than $x^{w*}$. For each fixed $\tau$, the generative IVP $\frac{dx_t(\tau)}{dt} = \frac{m_\tau - x_t(\tau)}{1-t}$ with $x_0(\tau) = x_0$ has the form solved in Proposition~1 with target $m_\tau$, yielding
\begin{equation}
    x_t(\tau) = t\, m_\tau + (1-t)\, x_0,
\end{equation}
whose terminal step is $x_1(\tau) = m_\tau$. At $\tau = 0$, $\beta_0 = 1$ gives $x_1(0) = x^l$, recovering the negative generation of the initial policy. Since $|1-\alpha| < 1$ and $\beta_{\tau\to \infty} \to 0$, we have
\begin{equation}
    \lim_{\tau \to \infty} x_1(\tau) = \lim_{\tau \to \infty} m_\tau = x^{w*}.
\end{equation}
This is an idealized field-space analysis: it describes how the terminal endpoint of a particular negative generation evolves when selected policy snapshots exhibit pointwise contraction toward a fixed conditional homing field. It is not intended to characterize the parameter-space dynamics of a finite neural network or to require the recurrence to hold after every SGD step. The index $\tau$ may instead enumerate selected snapshots at which this effective contraction model is used to describe the observed field evolution.
\end{proof}

\subsection{Reward Design}
\label{sec:reward_design}
We provide the reward model Gemini with the reference image and six canonical-view renders of one generated mesh. Following common VLM-as-a-judge practice, the prompt specifies the evaluator role, input evidence, assessment criteria, and output format. We use the following prompt template:

\begin{center}
\setlength{\fboxsep}{5pt}%
\fcolorbox{black}{gray!8}{%
\begin{minipage}{0.96\linewidth}
\small
\renewcommand{\baselinestretch}{1.08}\selectfont
\setlength{\parskip}{0.3em}
\textbf{Role.} You are a strict 3D asset reviewer. Given a reference image and six canonical-view renders of a generated mesh, evaluate the visible geometric evidence. Inspect all views before assigning a score.

\textbf{Reconstruction fidelity.} Compare the generated mesh against the reference. For silhouette proportions, check whether the mesh introduces unsupported parts that alter the outer contour, loses or breaks visible contour-defining parts, or poorly reconstructs salient regions such as the face, hands, or prominent accessories. For surface appearance, penalize image textures, color boundaries, or material patterns that are incorrectly converted into physical protrusions, grooves, or other geometry.

\textbf{Geometric plausibility.} Check whether structures that should be symmetric remain balanced. Penalize isolated structures with no support from the reference or the object's geometry. Across views, verify that detail density is coherent: unseen or occluded regions should contain neither implausibly elaborate additions nor conspicuously sparse details relative to corresponding visible regions.

\textbf{Orientation alignment.} Determine whether the mesh is aligned with the canonical world axes rather than inheriting the diagonal tilt, perspective skew, or camera pose of the reference image. A dominant object axis should appear horizontal or vertical in canonical views, and bilaterally symmetric objects should appear macroscopically balanced in the front view.

\textbf{Output.} Briefly justify each of the three dimensions above. First, assign each dimension an individual score from 1 to 5, where a higher score indicates better generation quality. Then, summarize the scores using a weighting ratio of Reconstruction Fidelity : Geometric Plausibility : Orientation Alignment \(= 3:2:1\), and return a single integer reward from 1 (severe defects) to 5 (high fidelity, plausible geometry, and canonical alignment). Output the final reward in the form \texttt{<answer>[1--5]</answer>}.
\end{minipage}}
\end{center}

As illustrated in Figure~\ref{fig:Reward_Design}, the reference image and six canonical views jointly provide evidence for evaluating fidelity, plausibility and orientation. Guided by the prompt template, the reward model first assesses reconstruction fidelity, geometric plausibility, and orientation alignment, and then aggregates the three assessments to produce an integer reward. The example demonstrates how the final reward is grounded in explicit observations of the generated geometry.

\begin{figure*}[!h]
    \centering
    \includegraphics[width=0.98\textwidth]{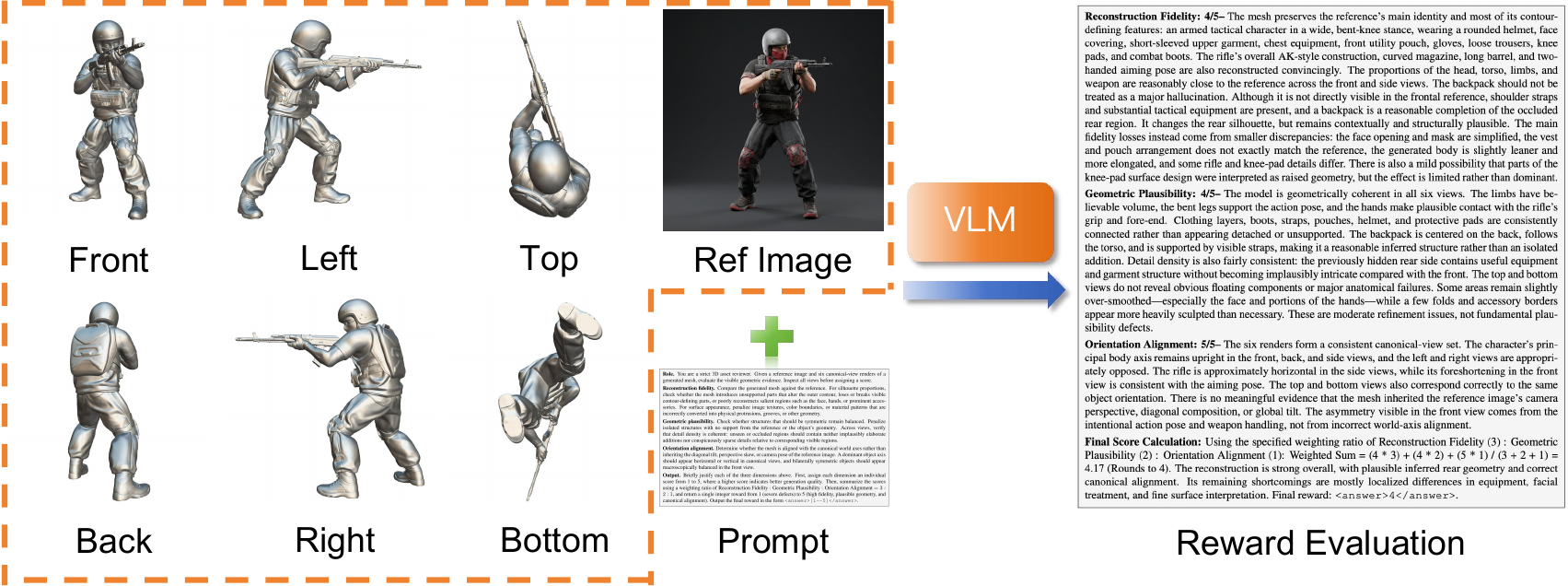}
    \caption{Illustration of the VLM-based reward evaluation process. The reward model receives a reference image, six canonical-view renderings of the generated mesh, and an evaluation prompt template, then produces criterion-wise assessments and a final scalar reward.}
    \label{fig:Reward_Design}
\end{figure*}


\begin{algorithm}[!h]
\caption{Dynamic Homing Optimization}
\label{alg:DHO}
\begin{algorithmic}[1]
\Require Policy $v_\theta$, reference policy $v_{\theta_{\text{ref}}}$, reward model $r(\cdot)$, CFG scale $s \ge 1$, weights $\eta \in (0,1)$ and $\beta > 0$, learning rate $\gamma$
\Ensure Optimized policy $v_\theta$
\Repeat
    \State $\hat{v}_\theta(x,t,c) \gets v_\theta(x,t,c_\varnothing) + s\big(v_\theta(x,t,c)-v_\theta(x,t,c_\varnothing)\big)$
    \State $\hat{v}_{\theta_{\text{ref}}}(x,t,c) \gets v_{\theta_{\text{ref}}}(x,t,c_\varnothing) + s\big(v_{\theta_{\text{ref}}}(x,t,c)-v_{\theta_{\text{ref}}}(x,t,c_\varnothing)\big)$
    \State Roll out $G$ samples $\{x_i\}_{i=1}^{G}$ for prompt $c$ using $\hat{v}_\theta$ and evaluate rewards $r_i \gets r(x_i,c)$
    \State Compute group-normalized advantages $A_i \gets (r_i-\bar{r})/s_r$
    \State Form balanced sets $\{(x_i^w,\omega_i^w)\}_{i=1}^{B}$ and $\{x_i^l\}_{i=1}^{B}$ for prompt $c$
    \State Sample times $t_i \sim \mathcal{U}(0,1)$
    \Statex \textit{Minimum-Cost Attractive Matching}
    \State Build $C_{i,j} \gets \|x_i^l - x_j^w\|_2^2$
    \State Solve $\boldsymbol{\pi}^{*} \gets \arg\min_{\boldsymbol{\pi}\in\Pi_B}\langle \mathbf{C}, \boldsymbol{\pi}\rangle$
    \State Assign $(x_i^{w*},\omega_i^{w*}) \gets (x_{\pi^*(i)}^{w},\omega_{\pi^*(i)}^{w})$
    \Statex \textit{Time-Aware Dynamic Correction}
    \For{$i \gets 1$ to $B$}
        \State $x_{t,i}^{w} \gets (1-t_i)x_{0,i}^{w} + t_i x_i^{w}$
        \State $x_{t,i}^{l} \gets (1-t_i)x_{0,i}^{l} + t_i x_i^l$
        \State $\tilde{v}_i \gets (x_i^{w*} - x_{0,i}^{l}) - t_i(x_i^l - x_{0,i}^{l})$
    \EndFor
    \Statex \textit{Policy Evaluation and Update}
    \Statex \textit{Positive samples use identity self-matching, $x_i^{w*}=x_i^w$.}
    \State $\mathcal{L}_{\text{pos}} \gets \frac{1}{B}\sum_i \omega_i^w\|\hat{v}_\theta(x_{t,i}^w, t_i, c) - (x_i^w - x_{0,i}^{w})\|_2^2$
    \State $\mathcal{L}_{\text{neg}} \gets \frac{1}{B}\sum_i \omega_i^{w*}\|(1-t_i)\hat{v}_\theta(x_{t,i}^l, t_i, c) - \tilde{v}_i\|_2^2$
    \State $\mathcal{L}_{\text{reg}} \gets \frac{1}{2B}\sum_i \sum_{x_t \in \{x_{t,i}^w, x_{t,i}^l\}} \|\hat{v}_\theta(x_t,t_i,c)-\hat{v}_{\theta_{\text{ref}}}(x_t,t_i,c)\|_2^2$
    \State $\theta \gets \theta - \gamma \nabla_\theta \big(\eta\mathcal{L}_{\text{neg}} + (1-\eta)\mathcal{L}_{\text{pos}} + \beta\mathcal{L}_{\text{reg}}\big)$
\Until{convergence}
\end{algorithmic}
\end{algorithm}

\subsection{DHO Pseudocode}
Algorithm~\ref{alg:DHO} summarizes the DHO optimization procedure. For each prompt, rollout samples are first ranked using group-normalized advantages. Positive- and negative-advantage samples form the positive and negative sets, respectively, and the larger set is subsampled to obtain balanced sets for one-to-one assignment. The retained positive samples are weighted by their normalized positive advantages, while each negative sample later inherits the weight of its assigned target. CFG is applied to both rollout generation and policy updates, with $s=1$ corresponding to the CFG-free setting. Note that Algorithm~\ref{alg:DHO} only presents the core DHO update, while Flow3D-Pro executes sample rollout and policy optimization asynchronously.

MAM computes distances in latent feature space between the two sets and solves a minimum-cost assignment using the Hungarian algorithm, yielding a distinct positive target for every negative sample. After sampling the flow time and constructing forward-process interpolants, TDC converts each assignment into a time-aware regression target that accounts for the remaining displacement to the positive endpoint. The positive branch retains the standard flow-matching target through self-assignment. Finally, DHO combines the advantage-weighted positive and negative losses with velocity-space regularization toward the frozen reference policy, and updates the policy using the resulting objective.


\subsection{Rollout Sampling}
Because DHO reconstructs policy-update interpolants from initial noises and rollout endpoints rather than stored sampling trajectories, its data collection is compatible with both ordinary differential equation (ODE) and stochastic differential equation (SDE) samplers. For each reference image, either sampler generates a group of $G$ meshes for reward evaluation.

\paragraph{ODE sampling.}
DHO adopts the noise-to-data convention
\begin{equation}
    x_t=(1-t)x_0+t x_1,
    \qquad
    x_0\sim\mathcal{N}(0,\mathbf I),
    \qquad
    d x_t=v_\theta(x_t,t,c)\,dt,
\end{equation}
where $t$ increases from $0$ (noise) to $1$ (data). For each reference image $c$, we independently draw $G$ initial noises $\{x_{0,i}\}_{i=1}^{G}$, where $x_{0,i}{\sim}\mathcal{N}(0,\mathbf I)$, and integrate the deterministic flow ODE from $t=0$ to $t=1$ using the same condition $c$. Since the ODE is deterministic given its initial state, diversity within the rollout group arises from the independent initial noises. The resulting endpoints $\{x_{1,i}\}_{i=1}^{G}$ are decoded into $G$ meshes.

\paragraph{SDE sampling.}
Inspired by Flow-GRPO~\cite{Flow-GRPO}, we convert the deterministic flow ODE into a stochastic process for rollout exploration. Flow-GRPO uses a data-to-noise time variable $\tau$ and performs generation using the reverse-time SDE
\begin{equation}
    d z_\tau
    =
    \left[
        \widetilde v_\theta(z_\tau,\tau,c)
        -
        \frac{\widetilde\sigma_\tau^2}{2}
        \nabla_{z_\tau}\log p_\tau(z_\tau\mid c)
    \right]d\tau
    +
    \widetilde\sigma_\tau\,dW_\tau,
    \qquad
    \widetilde\sigma_\tau
    =
    \kappa\sqrt{\frac{\tau}{1-\tau}},
\end{equation}
where $\kappa$ is a scalar hyper-parameter that controls the noise level. Under the change of variables
\begin{equation}
    t=1-\tau,
    \qquad
    x_t=z_{1-t},
    \qquad
    v_\theta(x_t,t,c)
    =
    -\widetilde v_\theta(z_\tau,\tau,c),
\end{equation}
the equivalent noise-to-data SDE used by DHO becomes
\begin{equation}
    d x_t
    =
    \left[
        v_\theta(x_t,t,c)
        +
        \frac{\sigma_t^2}{2}
        \nabla_{x_t}\log p_t(x_t\mid c)
    \right]dt
    +
    \sigma_t\,dW_t,
    \qquad
    \sigma_t=\kappa\sqrt{\frac{1-t}{t}}.
    \label{eq:dho-rollout-sde}
\end{equation}

For the rectified-flow path, the score can be recovered from the velocity as
\begin{equation}
    \nabla_{x_t}\log p_t(x_t\mid c)
    =
    \frac{t v_\theta(x_t,t,c)-x_t}{1-t}.
\end{equation}
Substituting it into Eq.~\eqref{eq:dho-rollout-sde} gives
\begin{equation}
    d x_t
    =
    \left[
        v_\theta(x_t,t,c)
        +
        \frac{\sigma_t^2}{2(1-t)}
        \left(
            t v_\theta(x_t,t,c)-x_t
        \right)
    \right]dt
    +
    \sigma_t\,dW_t.
\end{equation}

Using Euler--Maruyama discretization, we obtain the rollout update:
\begin{equation}
\begin{aligned}
x_{t+\Delta t}
={}&
x_t
+
\left[
v_\theta(x_t,t,c)
+
\frac{\sigma_t^2}{2(1-t)}
\left(
t\,v_\theta(x_t,t,c)-x_t
\right)
\right]\Delta t
+
\sigma_t\sqrt{\Delta t}\,\epsilon,
\\
={}&
x_t
+
\left[
v_\theta(x_t,t,c)
+
\frac{\kappa^2}{2t}
\left(
t\,v_\theta(x_t,t,c)-x_t
\right)
\right]\Delta t
+
\kappa\sqrt{\frac{1-t}{t}}\sqrt{\Delta t}\,\epsilon,
\end{aligned}
\label{eq:dho-rollout-update}
\end{equation}
where $\epsilon\sim\mathcal{N}(0,\mathbf I)$. We set $\kappa=0.2$ for SDE rollout experiments.


\section{Experiments}
\subsection{More Implementation Details}
\paragraph{Training and inference configurations for our method.} Stage I is based on the Hunyuan3D 2.1~\cite{hunyuan3d21} architecture. SFT uses a learning rate of $5\times10^{-6}$ and is trained for 5K steps on 32 H20 GPUs with a per-GPU batch size of 4. DHO is trained for 5K optimization steps with AdamW at a learning rate of $2.5\times10^{-6}$. Training is performed in bfloat16 with a per-GPU batch size of 2. We set the homing coefficient $\eta$ to $0.1$, the regularization coefficient $\beta$ to $1.0$, the rollout group size $G$ to $8$, and the CFG scale to $3.0$. For asynchronous training, we allocate $N=28$ GPUs to rollout and $M=4$ GPUs to policy updates, with $\lambda_{\min}=4$, $\lambda_{\max}=8$, and $K=16$. Although DHO supports both ODE- and SDE-based rollouts, in practice, we favor SDE sampling for greater rollout diversity. After DHO training, Stage I inference uses 50 ODE sampling steps. The output voxel representation of Stage I has a resolution of 518. Stage II is based on LATTICE~\cite{lai2025lattice} and uses a shape VAE to decode structure-conditioned latents into an implicit geometric field for mesh extraction. Meshes are extracted at a decoding resolution of 1536. It uses 5 ODE sampling steps and a CFG scale of $5.0$, and remains frozen for geometric refinement during inference.

\paragraph{RL objective comparison settings.} In the RL objective comparison experiments, for a fair comparison, all RL variants share the Stage I architecture, SFT initialization, reward model, and GenMesh-Test evaluation set, and their outputs are processed by the same frozen Stage II refiner. Our implementations of GRPO-style, DPO-style, and NFT-style methods follow Flow-GRPO~\cite{Flow-GRPO}, Flow-DPO~\cite{Flow-DPO, Qwen-Image-TR}, and DiffusionNFT~\cite{DiffusionNFT}, respectively. The positive and negative samples in DPO-style are paired using MAM, and $\beta_{\mathrm{DPO-style}}$ is set to 5000.0. For GRPO-style, we reparameterize Flow-GRPO's SDE from its original data-to-noise time convention into the equivalent noise-to-data formulation adopted in this work. We set the noise-level $\kappa_{\mathrm{GRPO-style}}$ to 0.2, matching the setting used in DHO. $\beta_{\mathrm{NFT-style}}$ is set to 1.0. All methods use asynchronous online training, with GPU configurations held constant across methods for both training and inference. Training is performed in bfloat16 with a per-GPU batch size of 2 and a rollout group size of $G=8$. Inference uses 50 ODE sampling steps. All RL objectives are trained for 5K optimization steps using AdamW with a learning rate of $2.5\times10^{-6}$. 


\subsection{User Study Evaluation Criteria}
For relative preference scoring, technical art experts are asked to sort all generated meshes from highest to lowest. The main considerations for the User Study are fidelity and plausibility. Fidelity mainly reflects how recognizably the generated mesh conveys the identity and overall proportions of the reference image, while plausibility mainly reflects whether the resulting shape reads as a coherent and well-formed 3D object in its own right. Experts make this judgment holistically from their professional experience rather than following a prescribed checklist. All evaluations are conducted in a blind manner: method identities are anonymized and the presentation order is randomized for every sample. Each reference image is independently evaluated by three experts, who rank the outputs of all methods from best to worst. The resulting ranking scores are averaged across the three experts and then across all reference images. Human experts rank the generated meshes of U methods from best (U - 1) to worst (0), revealing relative preference relationships. 


\subsection{Component Ablations}
This section provides additional quantitative and qualitative evaluations of the key components of DHO. For the component studies, we use the same Stage I architecture, training pipeline, reward model, and evaluation protocol, varying only the factor under study. 
Since each User Study produces a relative ranking of the variants within its own comparison group, these scores should be interpreted within each table or figure rather than across different ablations. We conduct a series of ablation studies on Stage I and evaluate the resulting models on GenMesh-Test.

\paragraph{Target assignment.}
We first isolate the target-selection mechanism in MAM. All training and evaluation settings are held fixed while MAM is replaced with random pairing, nearest-neighbor matching, or soft attraction. Adapted from the positive-sample attractive component of Drifting Models~\cite{Drifting_Model}, soft attraction uses a normalized similarity-weighted aggregation over all retained positive samples to produce a soft multi-target direction for each negative sample. This comparison examines whether minimum-cost one-to-one assignment provides a more effective homing target than unconstrained or soft diffuse alternatives.

\begin{figure*}[!h]
    \centering
    \includegraphics[width=0.95\textwidth]{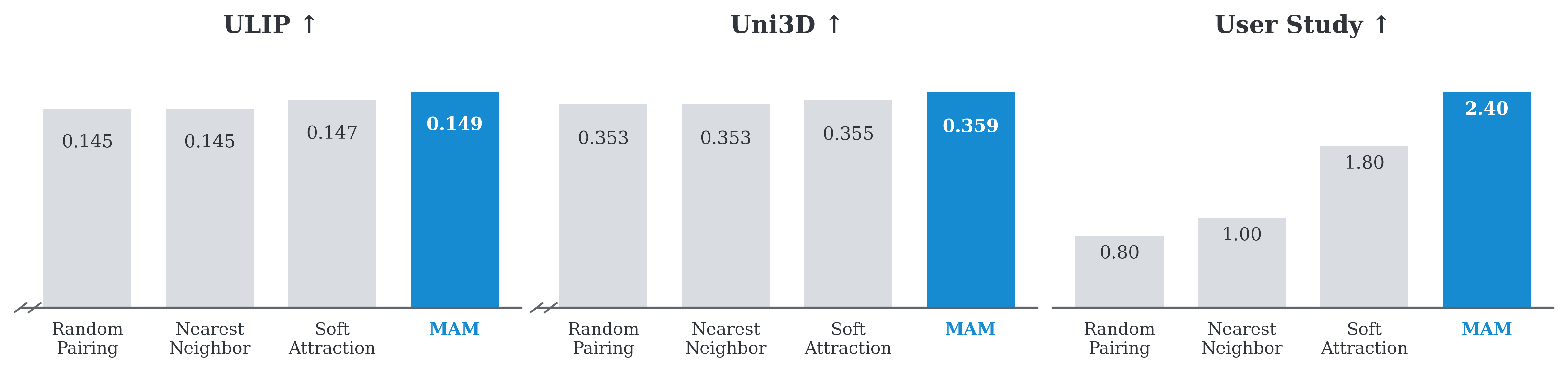}
    \caption{Effect of the target-assignment strategy in DHO.}
    \label{fig:MAM}
\end{figure*}

Figure~\ref{fig:MAM} shows that MAM obtains the highest scores on all three metrics, outperforming soft attraction as well as random and nearest-neighbor pairing. Soft attraction aggregates multiple positive samples and can retain a diffuse multi-target direction near equilibrium, whereas MAM assigns each negative sample a single positive target under a one-to-one constraint. These results support the use of deterministic, minimum-cost one-to-one assignment over unconstrained or soft multi-target guidance.

\paragraph{Sample selection.}
We evaluate whether prioritizing higher-advantage positive samples in MAM benefits DHO. When positives outnumber negatives, we compare retaining the $B$ highest-advantage positives with randomly selecting $B$ positives to form balanced sets for MAM.

\begin{figure*}[!h]
    \centering
    \includegraphics[width=0.80\textwidth]{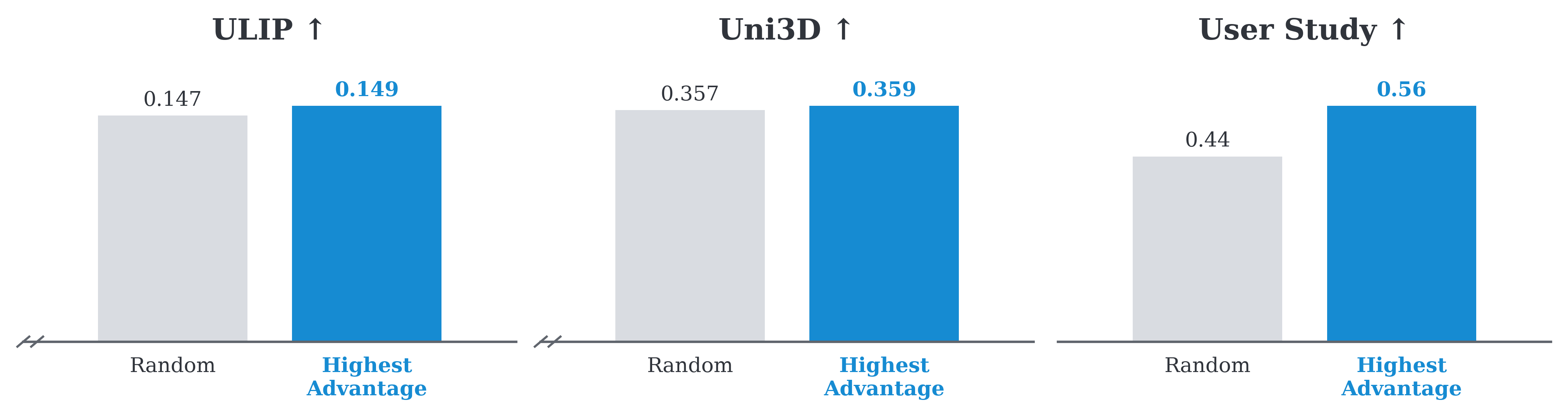}
    \caption{Effect of positive-sample selection in MAM.}
    \label{fig:Sample_Selection}
\end{figure*}

As shown in Figure~\ref{fig:Sample_Selection}, highest-advantage selection performs better than random selection, with User Study increasing from $0.44$ to $0.56$. The results support prioritizing higher-advantage positive samples as homing targets in MAM.

\paragraph{Time-aware correction.}
We next study the temporal weighting in TDC by comparing the unweighted objective, $w(t)=1$, with the proposed $w(t)=(1-t)^2$. The latter cancels the $1/(1-t)$ factor in the homing velocity when the loss is rewritten, yielding a numerically stable regression target over the full time horizon. All other components, including MAM and advantage weighting, are unchanged.

\begin{table}[!h]
  \centering
  \caption{Effect of time-aware weighting in TDC.}
  \label{tab:trajectory_stability}
  \resizebox{0.6\linewidth}{!}{
  \begin{tabular}{lccc}
    \toprule
    Time Weighting & ULIP $\uparrow$ & Uni3D $\uparrow$ & User Study $\uparrow$ \\
    \midrule
    $w(t) = 1$ & 0.146 & 0.354 & 0.45 \\
    $w(t) = (1-t)^2$ & \textbf{0.149} \textbf{\textcolor{blue}{(+2.05\%)}} & \textbf{0.359} \textbf{\textcolor{blue}{(+1.41\%)}} & \textbf{0.55} \\
    \bottomrule
  \end{tabular}
  }
\end{table}

As described in Table~\ref{tab:trajectory_stability}, the time-aware objective obtains higher scores than the unweighted variant on all reported metrics: ULIP increases from $0.146$ to $0.149$, Uni3D from $0.354$ to $0.359$, and User Study from $0.45$ to $0.55$. Within this controlled comparison, the result supports the empirical benefit of using the numerically stable TDC regression.

\paragraph{Dynamic correction.}
The benefit of dynamic correction is assessed by replacing the negative-branch homing velocity $(x^{w*}-x_t^l)/(1-t)$ with the static action $x^{w*}-x_0$, while keeping MAM and other training settings unchanged.

\begin{table}[!h]
  \centering
  \caption{Static versus dynamic velocity targets for the negative branch.}
  \label{tab:dynamic_correction}
  \resizebox{0.6\linewidth}{!}{
  \begin{tabular}{lccc}
    \toprule
    Velocity Target & ULIP $\uparrow$ & Uni3D $\uparrow$ & User Study $\uparrow$ \\
    \midrule
    Static & 0.145 & 0.353 & 0.31 \\
    Dynamic & \textbf{0.149} \textbf{\textcolor{blue}{(+2.76\%)}} & \textbf{0.359} \textbf{\textcolor{blue}{(+1.70\%)}} & \textbf{0.69} \\
    \bottomrule
  \end{tabular}
  }
\end{table}

Dynamic correction achieves higher scores on all three metrics in Table~\ref{tab:dynamic_correction}, raising User Study from $0.31$ to $0.69$. Accounting for the current state and remaining time therefore yields more effective guidance than a static action.

\paragraph{Advantage weighting.}
We evaluate whether the quality differences among positive samples should affect their contributions to the DHO objective. The weighted variant scales each positive-sample loss by its normalized positive advantage and transfers the same weight to negative samples assigned to that positive target. The unweighted variant treats all retained pairs equally. 

\begin{table}[!h]
  \centering
  \caption{Ablation of positive-sample advantage weighting in DHO.}
  \label{tab:advantage_weighting}
  \resizebox{0.65\linewidth}{!}{
  \begin{tabular}{cccc}
    \toprule
    Advantage Weighting & ULIP $\uparrow$ & Uni3D $\uparrow$ & User Study $\uparrow$ \\
    \midrule
    Without & 0.145 & 0.354 & 0.42 \\
    With & \textbf{0.149} \textbf{\textcolor{blue}{(+2.76\%)}} & \textbf{0.359} \textbf{\textcolor{blue}{(+1.41\%)}} & \textbf{0.58} \\
    \bottomrule
  \end{tabular}
  }
\end{table}

As shown in Table~\ref{tab:advantage_weighting}, advantage weighting performs better than the unweighted variant. This result suggests that emphasizing higher-advantage positive targets is beneficial in DHO.

\paragraph{Classifier-free guidance.}
We investigate how the CFG scale affects DHO. The same CFG scale is used consistently for rollout generation, policy updates, and inference. When the scale is set to $1.0$, the guided velocity reduces to the conditional velocity, so the unconditional prediction has no effect. When the scale is greater than $1.0$, the conditional prediction is extrapolated away from the unconditional prediction to increase conditioning strength.

\begin{figure*}[!h]
    \centering
    \includegraphics[width=0.95\textwidth]{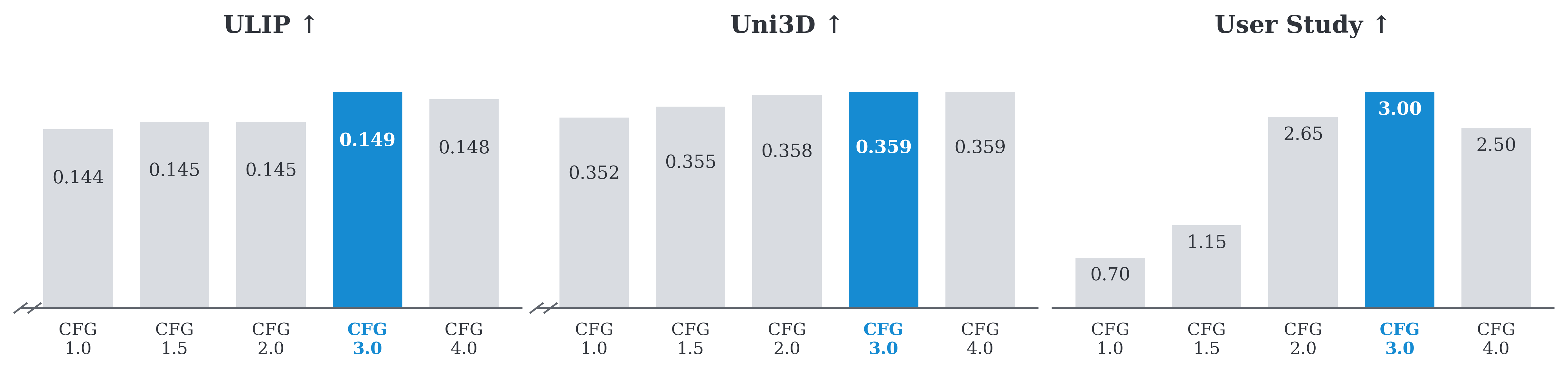}
    \caption{Ablation study on the CFG scale.}
    \label{fig:CFG}
\end{figure*}

As shown in Figure~\ref{fig:CFG}, using CFG with $s=3.0$ performs better than the CFG-free setting ($s=1.0$) and the other evaluated guidance scales. This result suggests that $s=3.0$ provides a favorable balance between rollout diversity and generation quality during policy optimization.

\paragraph{Negative homing and reference regularization.}
Moreover, we evaluate the balance among positive anchoring, negative homing, and reference regularization. We vary the homing coefficient $\eta$ and the regularization coefficient $\beta$ while retaining the remaining training configuration.

\begin{table}[!h]
  \centering
  \caption{DHO performance with the homing coefficient $\eta$ and the regularization coefficient $\beta$. The upper block fixes $\beta$ and varies $\eta$ (where $\eta=0.0$ disables the dynamic homing branch for negative-samples), while the lower block fixes $\eta$ and varies $\beta$ (where $\beta=0.0$ removes the reference-policy regularization). Each block is an independent User Study comparison group.}
  \label{tab:kl}
  \resizebox{0.75\linewidth}{!}{
  \begin{tabular}{cccc|ccc|ccc}
    \toprule
    \multicolumn{4}{c|}{$\eta$} & \multicolumn{3}{c|}{$\beta$} & \multirow{2}{*}{ULIP $\uparrow$} & \multirow{2}{*}{Uni3D $\uparrow$} & \multirow{2}{*}{User Study $\uparrow$} \\
    \cmidrule(lr){1-4} \cmidrule(lr){5-7}
    0.0 & 0.1 & 0.2 & 0.5 & 0.0 & 1.0 & 5.0 & & & \\
    \midrule
    \ding{51} &           &           &           &           & \ding{51} &           & 0.145 & 0.354 & 1.13 \\
              & \ding{51} &           &           &           & \ding{51} &           & \textbf{0.149} \textbf{\textcolor{blue}{(+2.76\%)}} & \textbf{0.359} \textbf{\textcolor{blue}{(+1.41\%)}} & \textbf{1.92} \\
              &           & \ding{51} &           &           & \ding{51} &           & 0.147 \textcolor{blue}{(+1.38\%)} & 0.357 \textcolor{blue}{(+0.85\%)} & 1.68 \\
              &           &           & \ding{51} &           & \ding{51} &           & 0.145 \textcolor{blue}{(+0.00\%)} & 0.355 \textcolor{blue}{(+0.28\%)} & 1.27 \\
    \midrule
              & \ding{51} &           &           & \ding{51} &           &           & 0.140 & 0.342 & 0.30 \\
              & \ding{51} &           &           &           & \ding{51} &           & \textbf{0.149} \textbf{\textcolor{blue}{(+6.43\%)}} & \textbf{0.359} \textbf{\textcolor{blue}{(+4.97\%)}} & \textbf{1.70} \\
              & \ding{51} &           &           &           &           & \ding{51} & 0.146 \textcolor{blue}{(+4.29\%)} & 0.358 \textcolor{blue}{(+4.68\%)} & 1.00 \\
    \bottomrule
  \end{tabular}
  }
\end{table}

As shown in Table~\ref{tab:kl}, among the tested settings, $\eta=0.1$ and $\beta=1.0$ obtain the highest scores and are therefore used as the default configuration. Since $\eta$ weights $\mathcal{L}_{\text{neg}}$, setting $\eta=0.0$ removes the dynamic homing of negative samples, leaving only the positive-sample branch. This variant drops from $0.149$/$0.359$/$1.92$ to $0.145$/$0.354$/$1.13$, indicating that the gains do not come from the positive-sample branch alone. However, increasing $\eta$ to $0.5$ reduces performance, suggesting that negative homing should be appropriately balanced with positive-sample anchoring. Removing the reference-policy regularization ($\beta=0.0$) yields worse results, where the policy deviates from the pretrained velocity field and tends to produce degenerate geometry. These results suggest that dynamic homing of negative samples supplies the corrective signal driving the improvement, while the reference-policy regularization keeps it within a stable neighborhood of the pretrained policy.

\paragraph{Rollout group size.}
The rollout group size $G$ determines how many candidate meshes are generated for each reference image before computing group-normalized advantages. We study its effect by varying group sizes while keeping the Stage I architecture and DHO objective unchanged.

\begin{figure*}[!h]
    \centering
    \includegraphics[width=0.95\textwidth]{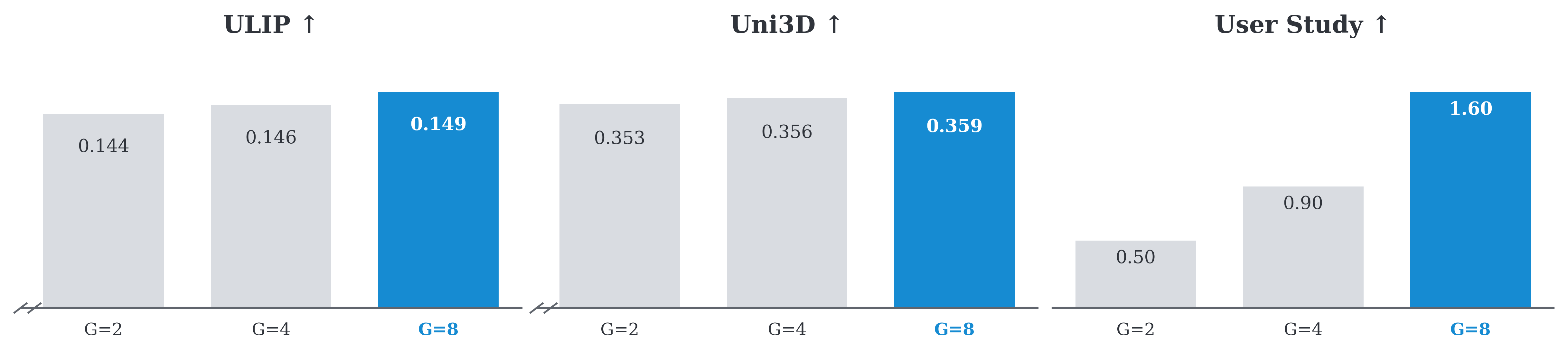}
    \caption{Effect of the rollout group size $G$ on DHO performance.}
    \label{fig:Group}
\end{figure*}

As shown in Figure~\ref{fig:Group}, performance improves consistently as $G$ increases from $2$ to $8$, with the User Study score rising from $0.50$ to $1.60$. This suggests that a larger candidate set enables broader exploration, provides more reliable relative preference signals, and yields a richer set of positive anchors for MAM. These results suggest the potential of larger rollout groups to improve DHO performance.

\paragraph{Rollout sampler.}
We compare different rollout samplers, including ODE and SDE, to evaluate their influence on DHO performance. Both variants generate $G=8$ candidate meshes for each reference image, while the SDE sampler additionally injects stochasticity along each sampling trajectory with $\kappa=0.2$. 

\begin{table}[!h]
  \centering
  \caption{Comparison of ODE and SDE rollout samplers.}
  \label{tab:rollout_sampler}
  \resizebox{0.6\linewidth}{!}{
  \begin{tabular}{cccc}
    \toprule
    Rollout Sampler & ULIP $\uparrow$ & Uni3D $\uparrow$ & User Study $\uparrow$ \\
    \midrule
    ODE & 0.145 & 0.353 & 0.33 \\
    SDE & \textbf{0.149} \textbf{\textcolor{blue}{(+2.76\%)}} & \textbf{0.359} \textbf{\textcolor{blue}{(+1.70\%)}} & \textbf{0.67} \\
    \bottomrule
  \end{tabular}
  }
\end{table}

As shown in Table~\ref{tab:rollout_sampler}, SDE rollout performs better than ODE sampling, especially increasing User Study from $0.33$ to $0.67$. Compared with deterministic ODE sampling, the additional trajectory-level stochasticity of SDE broadens exploration. 

\begin{figure*}[!h]
    \centering
    \includegraphics[width=1.0\textwidth]{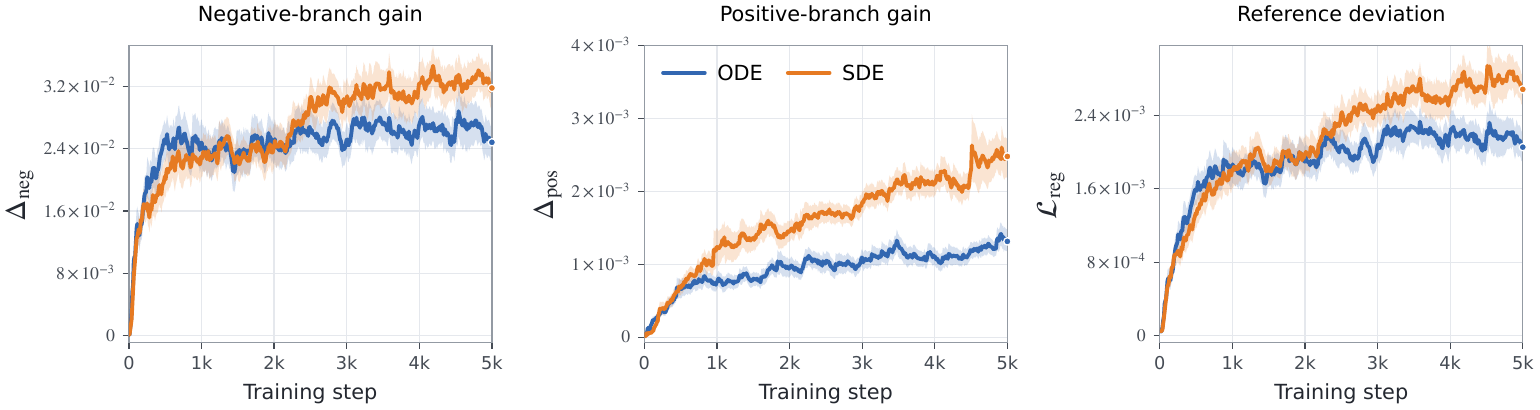}
    \caption{DHO training dynamics with ODE and SDE rollouts. From left to right: negative-branch gain, positive-branch gain, and reference-policy deviation.}
    \label{fig:Training_Curve}
\end{figure*}

In Figure~\ref{fig:Training_Curve}, for negative samples, $\Delta_{\mathrm{neg}}$ measures how much the current policy reduces the mean squared error (MSE) to the homing target compared with the frozen reference. For positive samples, $\Delta_{\mathrm{pos}}$ measures the corresponding reduction for the positive-sample target. Larger values indicate greater improvements in target fitting. $\mathcal{L}_{\mathrm{reg}}$ is the mean squared policy--reference velocity difference across both branches. SDE yields larger gains in both branches.


\paragraph{Reward gain.}
The contribution of each reward dimension to the geometric quality is further evaluated. Starting from the same SFT initialization, we progressively enable reconstruction fidelity, geometric plausibility, and orientation alignment in the reward prompt, and post-train Stage I with DHO. 

\begin{table}[!h]
  \centering
  \caption{Impact of the reward dimensions used for DHO post-training, where ``--'' denotes the SFT model.}
  \label{tab:reward_gain}
  \resizebox{0.55\linewidth}{!}{
  \begin{tabular}{ccc|c}
    \toprule
    \makecell{Reconstruction\\Fidelity} & \makecell{Geometric\\Plausibility} & \makecell{Orientation\\Alignment} & User Study $\uparrow$ \\
    \midrule
    --        & --        & --        & 0.63 \\
    \ding{51} &           &           & 1.46 \\
    \ding{51} & \ding{51} &           & 1.85 \\
    \ding{51} & \ding{51} & \ding{51} & \textbf{2.06} \\
    \bottomrule
  \end{tabular}
  }
\end{table}

As presented in Table~\ref{tab:reward_gain}, reconstruction fidelity yields the largest gain, with further improvements from progressively incorporating geometric plausibility and orientation alignment. 


\paragraph{Runtime analysis.} At inference time, Stage I takes an average of 62~s per sample on an NVIDIA H20 GPU. This runtime includes image loading and preprocessing, conditional feature extraction, 50-step ODE sampling, VAE-based mesh decoding, and mesh conversion.

Stage II geometry refinement takes an average of 81~s per sample on the same GPU using 5-step ODE sampling. This runtime includes image loading and preprocessing, DINO-based conditional feature extraction, input mesh loading and voxelization, flow-matching sampling, VAE-based implicit-field decoding, and mesh conversion.

\subsection{More Quantitative Comparisons with Current Methods}
\paragraph{Comparison with 3D generators on LATTICE-Bench.} We conduct additional experiments with current 3D geometry generation methods on LATTICE-Bench~\cite{lai2025lattice}, using ULIP and Uni3D similarity metrics.

\begin{table}[!h]
  \centering
  \caption{Quantitative evaluation on the LATTICE-Bench.}
  \label{tab:lattice}
  \resizebox{0.5\linewidth}{!}{
  \begin{tabular}{lcc}
    \toprule
    Method & ULIP $\uparrow$ & Uni3D $\uparrow$ \\
    \midrule
    Michelangelo~\cite{zhao2023michelangelo} & 0.115 & 0.261 \\
    Craftsman 1.5~\cite{li2024craftsman3d} & 0.129 & 0.298 \\
    Trellis~\cite{xiang2024trellis} & 0.126 & 0.311 \\
    Hunyuan3D 2.0~\cite{hunyuan3d20} & 0.130 & 0.315 \\
    Hi3DGen~\cite{Hi3dgen} & 0.112 & 0.299 \\
    Direct3D-s2~\cite{Direct3D-S2} & 0.122 & 0.314 \\
    LATTICE~\cite{lai2025lattice} & 0.130 & 0.315 \\
    Flow3D-OPD~\cite{Flow3D-OPD} & 0.138 & 0.321 \\
    Flow3D-Pro & \textbf{0.140} & \textbf{0.326} \\
    \bottomrule
  \end{tabular}
  }
\end{table}

Consistent with the results on GenMesh-Test, Table~\ref{tab:lattice} indicates that, after post-training with DHO, Flow3D-Pro obtains the highest scores among the listed representative image-to-3D mesh generation methods, reaching 0.140 in ULIP and 0.326 in Uni3D. These results demonstrate the role of DHO in improving geometric quality.

\paragraph{Comparison with additional post-training methods.}
We conduct experiments at Stage I, comparing SFT, rejection fine-tuning (RFT), PFM~\cite{kim2024preferencealignmentflowmatching}, and DHO on GenMesh-Test to evaluate DHO against positive-only fine-tuning and auxiliary preference-flow alignment. RFT fine-tunes the SFT model only on positive-advantage rollouts, whereas PFM trains an auxiliary latent-space flow from negative to positive samples while keeping the SFT generator frozen. The preference-based variants share the training images and reward model. Base-generator inference settings and mesh decoding are held fixed, with PFM additionally integrating its preference flow.

\begin{table}[!h]
  \centering
  \caption{Comparison with additional post-training methods.}
  \label{tab:additional_post_training_methods}
  \resizebox{0.65\linewidth}{!}{
  \begin{tabular}{lccc}
    \toprule
    Method & ULIP $\uparrow$ & Uni3D $\uparrow$ & User Study $\uparrow$ \\
    \midrule
    SFT & 0.144 & 0.351 & 0.97 \\
    RFT & 0.145 \textcolor{blue}{(+0.69\%)} & 0.355 \textcolor{blue}{(+1.14\%)} & 1.39 \\
    PFM~\cite{kim2024preferencealignmentflowmatching} & 0.147 \textcolor{blue}{(+2.08\%)} & 0.356 \textcolor{blue}{(+1.42\%)} & 1.70 \\
    DHO & \textbf{0.149} \textbf{\textcolor{blue}{(+3.47\%)}} & \textbf{0.359} \textbf{\textcolor{blue}{(+2.28\%)}} & \textbf{1.94} \\
    \bottomrule
  \end{tabular}
  }
\end{table}

As shown in Table~\ref{tab:additional_post_training_methods}, DHO outperforms both RFT and PFM. These results suggest that directly correcting the original generative field toward matched positive targets improves geometric quality more effectively than positive-only fine-tuning or auxiliary preference-flow alignment.

\begin{figure*}[!h]
    \centering
    \includegraphics[width=0.75\textwidth]{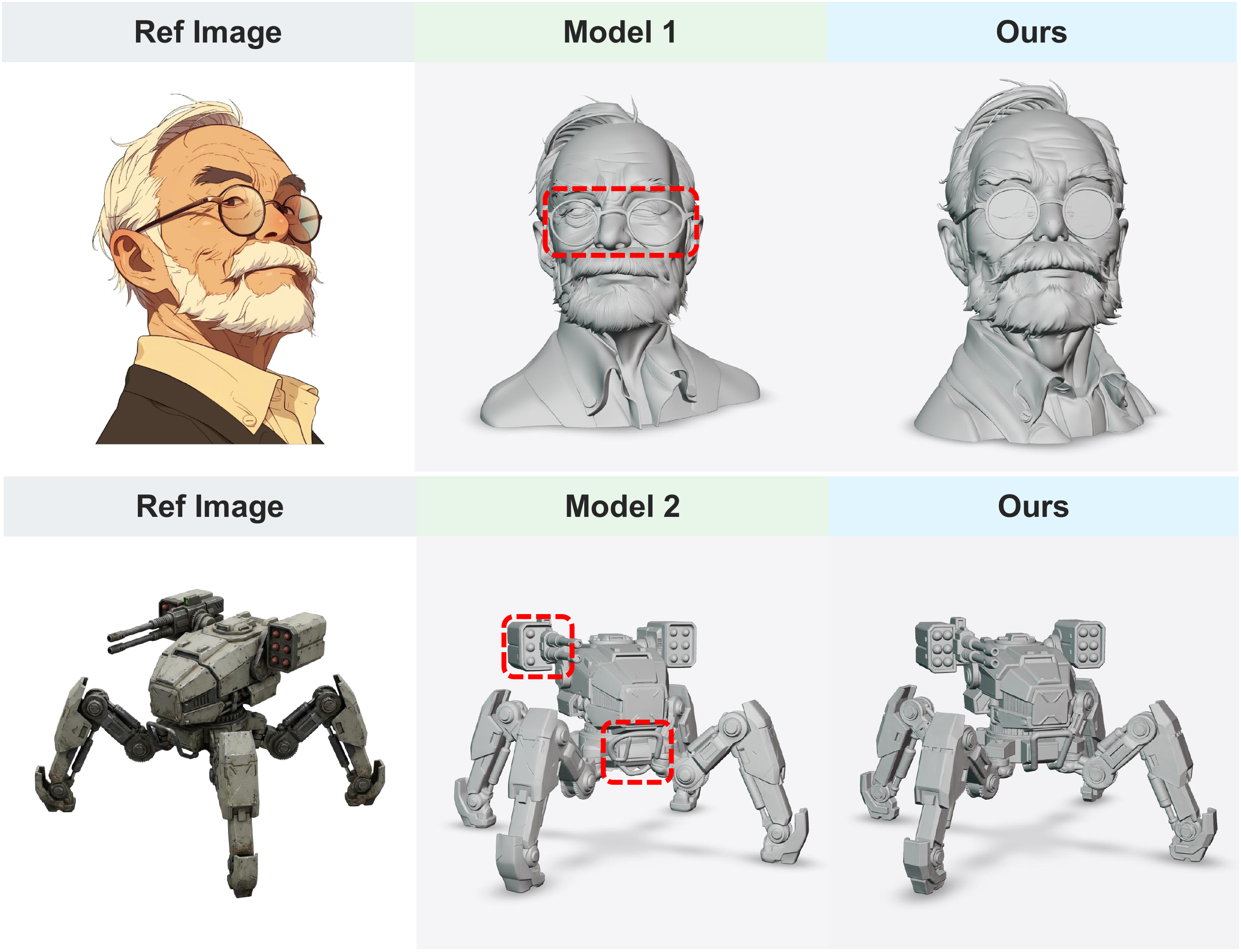}
    \caption{Visual comparison against commercial models. To facilitate visualization, all generated meshes are rotated to a common orientation.}
    \label{fig:Commercial_Comparison}
\end{figure*}

\subsection{More Qualitative Results}
\paragraph{Commercial model comparison.} Figure~\ref{fig:Commercial_Comparison} compares Flow3D-Pro with two closed-source commercial methods, denoted as Model 1 and Model 2. In the shown examples, Model 1 exhibits a skewed nose and eyeglasses with missing lenses, while Model 2 shows inconsistent left--right launcher layouts and distorted central geometry. Flow3D-Pro better preserves facial symmetry and bilateral component consistency--benefiting from the orientation-alignment and geometric-plausibility rewards, which jointly facilitate learning symmetric structures--while achieving more plausible global shapes and higher reconstruction fidelity.

\paragraph{More generation results.} We provide additional visual results to evaluate Flow3D-Pro on diverse reference images. Figure~\ref{fig:More_Visualization} covers a broad range of categories, including human and stylized characters, animals, plants, vehicles, architecture, everyday objects, and hard-surface assets. Each panel displays the reference image in the upper-right corner and the generated mesh in the main view.

\begin{figure*}[!h]
    \centering
    \includegraphics[width=1.0\textwidth]{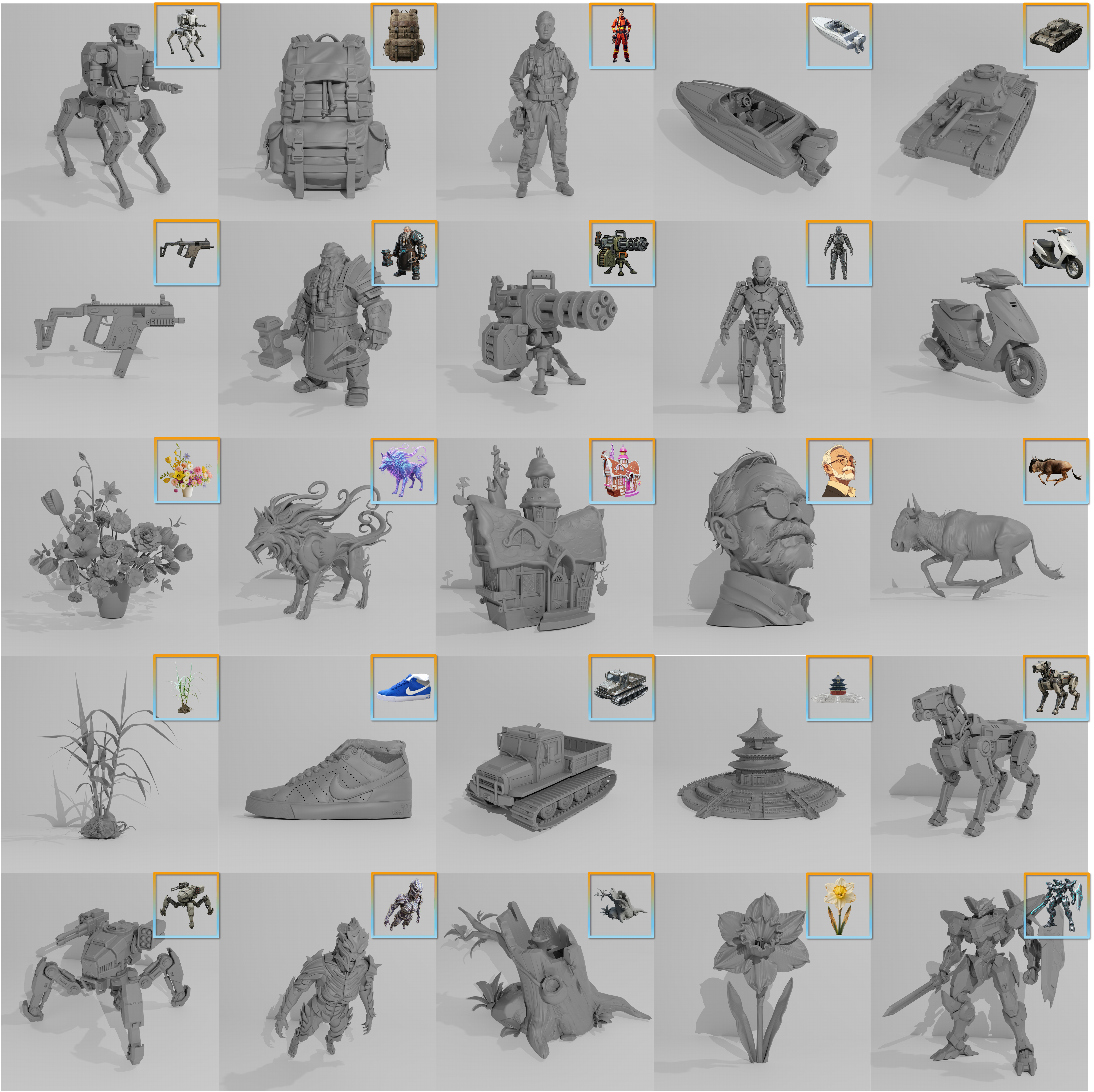}
    \caption{More qualitative results of Flow3D-Pro across diverse object categories and geometric structures.}
    \label{fig:More_Visualization}
\end{figure*}

Across the displayed examples, Flow3D-Pro reconstructs recognizable global silhouettes and category-specific structures while retaining local details such as character garments and facial features, articulated mechanical components, vehicle structures, and foliage. Together, these examples complement the quantitative evaluations and demonstrate Flow3D-Pro's versatility across diverse object categories and geometric structures.

\paragraph{Video demo.} The supplementary materials include a video showcasing additional visual results of our method.

\section{Limitations and Future Work}

Designing a reward that comprehensively characterizes the geometric quality of a reference-image-guided 3D mesh remains difficult, and our current dimensions are not fine-grained enough for higher-precision geometric properties to be accurately reflected. We plan to further decompose the reward dimensions and collect feedback from more professional artists to refine the reward design and improve generation quality.

In addition, DHO is applied only to the lifting stage while the refinement stage remains frozen, since the latter is tightly conditioned on the Stage I output and its rollouts therefore exhibit limited diversity for policy optimization. We plan to explore ways of enriching the rollout diversity of the Stage II.

In current 3D generation, constrained by pretrained model capabilities, rollout diversity, and reward-distribution complexity, DHO with positive-sample attraction-guided dynamic homing achieves better mesh post-training quality than representative DPO-, GRPO-, and NFT-style objectives, which steer predicted velocities away from negative-trajectory directions without explicitly specifying a target velocity field toward preferred samples. Since 2D visual generation differs from the 3D setting in pretrained priors and reward distributions, we plan to carry out further exploratory experiments to examine the performance of DHO in 2D visual generation.

\end{document}